%% file: iclr2026_conference.tex
\documentclass{article} 
\usepackage{iclr2026_conference,times}

\input{math_commands.tex}

\usepackage{hyperref}
\usepackage{url}
\usepackage{amsthm}
\usepackage{algorithm}
\usepackage{algpseudocode}
\usepackage{booktabs}
\usepackage{multirow}
\usepackage{graphicx}
\usepackage{wrapfig}
\usepackage{enumitem}
\usepackage{caption}
\usepackage{pifont}
\usepackage[table]{xcolor}
\usepackage{color}
\usepackage{soul}

\newtheorem{theorem}{Theorem}[section]
\newtheorem{corollary}[theorem]{Corollary}

\newcommand{\rowidx}{\stepcounter{rownum}\arabic{rownum}}

\title{Amortized Moment Matching for Visual\\ Generation}

\author{
    Wenze Liu$^{1,}$\thanks{Work done at Kuaishou Technology.} \quad
    Xintao Wang$^{2}$ \quad
    Pengfei Wan$^{2}$ \quad
    Xiangyu Yue$^{1,}$\thanks{Corresponding author.} \\
    \\
    $^1$MMLab, CUHK \\
    $^2$Kling Team, Kuaishou Technology
}

\iclrfinalcopy 
\begin{document}

\maketitle
\lhead{Preprint}
\begin{abstract}
We propose amortized moment matching, utilizing neural networks to learn data moments as distributional training signals. By casting diffusion denoisers through polynomial projections, we establish a general framework for moment amortization, revealing that an $n$-th degree projection explicitly identifies data moments up to order $n+1$. Derived from the tractable affine case, we instantiate the Amortized Fr\'{e}chet Distance (AMFD) loss. Unlike FD-loss which relies on explicit marginal moment calculations, AMFD is able to dynamically learn conditional moments via an alternating, matrix-free optimization pipeline that effortlessly scales to high-dimensional data. When operating on global representation features, AMFD serves as a powerful post-training objective; empirically, its neural formulation yields more robust training dynamics than exact statistical matching, substantially surpassing the FD baseline on the FDr$^6$ metric and achieving superior one-step generation on ImageNet. Furthermore, it unlocks direct exploration within native generative spaces, suggesting that the first two moments can identify target distributions only in spaces with strong semantics. Finally, when scaled to text-to-image generation, the condition-aware nature of AMFD unlocks massive gains in instruction-following capabilities, enabling our one-step models to outperform their multi-step FLUX.2 [klein] 4B teachers on the GenEval benchmark while achieving on-par performance on PickScore. Code and checkpoints are available at \url{https://github.com/poppuppy/amfd}.
\end{abstract}

\section{Introduction}
\begin{wrapfigure}[18]{r}{0.43\textwidth}
\vspace{-16pt}
  \centering
  \includegraphics[width=\linewidth]{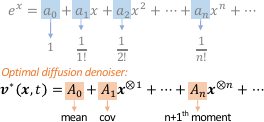}
  \caption{\textbf{Diffusion denoiser polynomial projection.} The Maclaurin series expands $e^x$ into scalar polynomial terms; analogously, expanding the optimal diffusion denoiser via tensor powers yields high-dimensional coefficients. We show that these coefficients explicitly map to the underlying distributional moments, where $A_0$ corresponds to the mean, $A_1$ to the covariance, and $A_n$ to the $(n+1)$-th moment.}
  \label{fig:projection}
\end{wrapfigure}
Distribution matching is a natural paradigm for generative modeling, aiming to directly align the generated distribution with the target distribution. Classical approaches adopt moment matching and maximum mean discrepancy (MMD) objectives~\citep{sriperumbudur2010hilbert,gretton2012kernel,li2015generative,dziugaite2015training,mathiasen2020backpropagating}. However, they struggled to become competitive for visual generation, largely due to challenges in finite-sample estimation, kernel selection, and optimization stability~\citep{li2017mmd}. In contrast, dominant paradigms such as diffusion and flow matching~\citep{sohl2015deep,song2019generative,ho2020denoising,liu2022flow,lipman2022flow}, autoregressive models~\citep{van2016pixel,razavi2019generating,esser2021taming,sun2024autoregressive}, GANs~\citep{goodfellow2014generative,karras2019style}, and normalizing flows~\citep{dinh2016density,kingma2018glow,zhai2024normalizing} circumvent explicit distribution matching by adopting particle-level generative rules, allowing distributional alignment to emerge implicitly.

Recent progress, however, reveals that the historical weakness of direct distribution matching stemmed largely from the underlying data spaces. Once shifted to semantically rich representation spaces, distribution matching emerges as an effective paradigm. For instance, Drifting models~\citep{deng2026generative} utilize kernel forces to evolve generated distributions, while the FD-loss~\citep{yang2026representation} directly optimizes mean and covariance matching with a large population. Advancing the perspective of moment matching, we propose \emph{amortized moment matching}. Rather than estimating moments purely from empirical batches or fixed global references, we amortize the statistics themselves: a neural amortizer predicts moment operators as dynamic functions of the conditioning signal and/or the target representation space. Our theoretical motivation stems from diffusion models, where a denoising model learns to approximate the conditional expectation of the clean data or the velocity field given the noisy observation. As illustrated in Figure~\ref{fig:projection}, we analyze this denoiser function through polynomial projections: the affine projection recovers mean and covariance information, while higher-degree projections correspond to amortizing higher-order moments. This formulation offers a generalized route to moment amortization.

\begin{figure}[t]
    \centering
    \includegraphics[width=\linewidth]{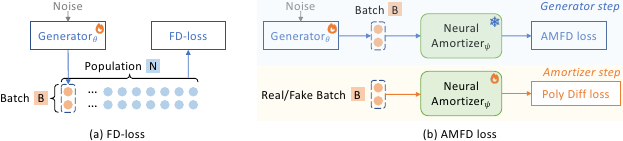}
    \caption{\textbf{Comparison between FD-loss and our AMFD loss.} (a) FD-loss explicitly computes marginal distribution statistics over a large population, propagating gradients through mini-batches. (b) Our AMFD loss uses neural amortizers which can leverage neural generalization to naturally encode conditions. Through an alternating training scheme, these amortizers dynamically maintain the conditional mean and covariance statistics, which the generator directly queries to obtain gradient signals for updating.}
    \label{fig:method}
    \vspace{-12pt}
\end{figure}

While high-order moment tensors grow exponentially in size, we instantiate a tractable and powerful second-order case termed the \emph{Amortized Fr\'{e}chet Distance} (AMFD) loss. When operating on global representation features, AMFD serves as an effective post-training objective that can enhance generative models and transform multi-step models into one-step generators. Crucially, in contrast to FD-loss~\citep{yang2026representation} which computes aggregated marginal moments, AMFD injects the conditioning signal directly into the amortizer. This aligns naturally with modern diffusion models, where denoisers operate as conditional functions of labels or text prompts, enabling dynamic, condition-aware distribution matching. As illustrated in Figure~\ref{fig:method}(b), AMFD formulates this as a fully deep-learning-native pipeline driven by an alternating optimization between the generator and the neural amortizer. To avoid the prohibitive complexity of explicit covariance matrices, we parameterize the affine function via Jacobian-Vector-Product (JVP) actions. This matrix-free efficiency effortlessly scales to high dimensions, which additionally unlocks the exploration of whether matching merely the first two moments provides sufficient generative signals within native generative spaces (\textit{e.g.}, raw pixels, VAE/RAE latents~\citep{rombach2022high,zheng2025diffusion}), where FD-loss is computationally intractable. Ultimately, AMFD strategically trades strict empirical exactness for superior conditional generalization and scalability.

\begin{figure}[!t]
    \centering
    \includegraphics[width=\linewidth]{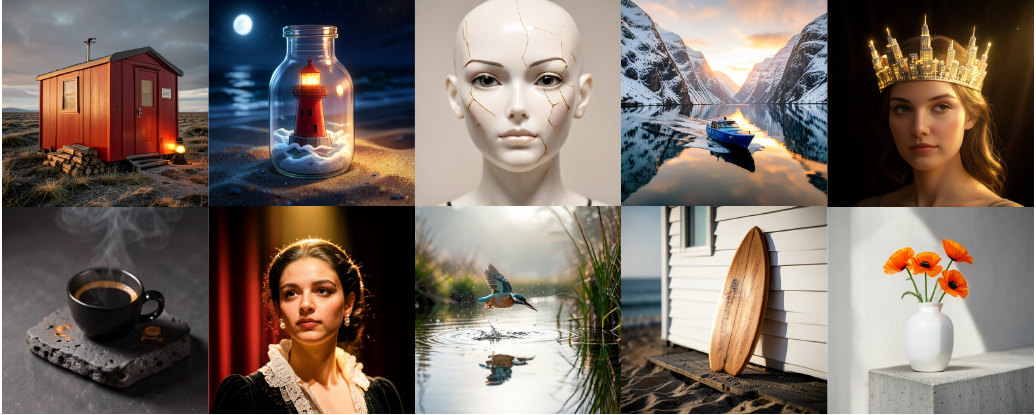}
    \caption{One-step text-to-image samples generated by our AMFD post-trained FLUX.2 [klein] 4B.}
    \label{fig:t2i-samples}
\end{figure}

We evaluate AMFD in diverse visual generation tasks. First, in ImageNet-256$\times$256 post-training, its neural formulation yields more robust optimization dynamics than the FD-loss baseline, achieving state-of-the-art FDr$^6$ scores of $1.79$ on JiT-H/16 and $1.75$ on pMF-H/16. Second, exploration within native generative spaces suggests that the first two moments can identify target distributions only in spaces with strong semantics, such as the RAE space. Finally, when scaled to text-to-image generation with pixel-space PixelGen~\citep{ma2026pixelgen} and latent-space FLUX.2 [klein] 4B~\citep{blackforestlabs2026flux2klein}, the condition-aware nature of AMFD unlocks massive gains in instruction-following capabilities. Notably, our one-step post-trained FLUX models significantly surpass their multi-step and few-step teachers on the GenEval benchmark, while maintaining on-par on PickScore.

Our main contributions are summarized as follows:
\begin{itemize}[leftmargin=*, itemsep=0pt, topsep=0pt, parsep=0pt]
\item For the first time, we establish a rigorous connection between diffusion denoisers and data moments by proving that an $n$-th degree polynomial projection of a denoiser explicitly identifies moments up to order $n+1$. Grounded in this theory, we introduce \emph{amortized moment matching}, a scalable framework that learns and matches distributional moments via neural amortizers.
\item We instantiate the tractable second-order case as the Amortized Fr\'{e}chet Distance (AMFD) loss. Unlike FD-loss that is restricted to unconditional marginals, AMFD leverages neural generalization to efficiently match fine-grained conditional moments---a capability crucial for scaling to complex text-to-image (T2I) generation.
\item Empirically, AMFD achieves state-of-the-art FDr$^6$ for one-step generation on ImageNet-256$\times$256, reaching $1.79$ (JiT-H/16) and $1.75$ (pMF-H/16). Furthermore, when extended to text-to-image generation, our one-step FLUX.2 models surpass their multi-step teachers on the GenEval benchmark ($0.846$ vs. $0.794$) while maintaining on-par performance on PickScore.
\end{itemize}

\section{Preliminaries}
\label{sec:preliminaries}

\subsection{Fr\'{e}chet Distance}
Let $P_r$ and $P_g$ denote the real and generated data distributions over the image space $\mathcal{X}$. Given a frozen representation encoder $\phi:\mathcal{X}\to\mathbb{R}^D$, we aim to compare the push-forward feature distributions $\phi_\# P_r$ and $\phi_\# P_g$. We define their respective first two moments for $i \in \{r, g\}$ as:
\begin{equation}
\label{eq:feature-moments}
\mu_i = \mathbb{E}_{X\sim P_i}[\phi(X)], \quad \Sigma_i = \operatorname{Cov}_{X\sim P_i}[\phi(X)],
\end{equation}
where $X$ denotes the random variable in $\mathcal{X}$. The squared Fr\'{e}chet distance, which measures the discrepancy between their Gaussian approximations, is defined as:
\begin{equation}
\label{eq:frechet-distance}
d_{\mathrm{FD}}^2(P_g,P_r;\phi) = \|\mu_g-\mu_r\|_2^2 + \operatorname{Tr}\left(\Sigma_g+\Sigma_r-2(\Sigma_g^{1/2}\Sigma_r\Sigma_g^{1/2})^{1/2}\right).
\end{equation}
When $\phi$ is Inception-v3~\citep{szegedy2016rethinking}, this recovers the widely used Fr\'{e}chet Inception Distance (FID)~\citep{heusel2017gans}. Recently, the FD-loss~\citep{yang2026representation} proposed directly optimizing Eq.~\ref{eq:frechet-distance} across multiple representation spaces to train generative models.

\subsection{Diffusion Models}
While our theoretical framework applies broadly to diffusion models, we adopt the flow matching formulation with linear interpolants~\citep{liu2022flow,lipman2022flow,albergo2023stochastic} for clarity. Consider a base distribution $X_0\sim P_0=\mathcal{N}(0,I)$ and a target conditional data distribution $X_1\sim P_r(\cdot\mid c)$. The linear interpolant is defined as $X_t=(1-t)X_0+tX_1$ for $t\in[0,1]$.

Flow matching trains a neural vector field $v_\theta$ by minimizing the regression objective:
\begin{equation}
\label{eq:flow-matching-loss}
\mathcal{L}_{\mathrm{FM}}(\theta) = \mathbb{E}\left[\left\| v_\theta(X_t,t,c)-(X_1-X_0)\right\|_2^2\right],
\end{equation}
which yields the population-level optimum:
\begin{equation}
\label{eq:flow-matching-optimum}
v^\star(x,t,c) = \mathbb{E}[X_1-X_0\mid X_t=x,c].
\end{equation}
This optimal vector field drives the probability flow ordinary differential equation (ODE)
\begin{equation}
\label{eq:flow-ode}
\frac{dX_t}{dt}=v^\star(X_t,t,c),
\end{equation}
which deterministically transports the prior distribution $P_0$ to the target distribution $P_r(\cdot\mid c)$.

\section{Amortized Moment Matching: Theoretical Motivation}
\label{sec:theory}

While the flow-matching objective Eq.~\ref{eq:flow-matching-loss} is conventionally interpreted as learning a transport vector field, we demonstrate from a distributional perspective that its optimal vector field (Eq.~\ref{eq:flow-matching-optimum}) fundamentally encodes conditional data statistics. Specifically, by analyzing the orthogonal projection of this denoiser onto polynomial function classes, we reveal a natural moment hierarchy. This connection allows us to isolate the affine case to derive a computationally tractable second-order matching objective. We formalize this insight below (see Appendix~\ref{sec:proof} for the rigorous formulation and proof).

\begin{theorem}[Denoising Projection Moment Hierarchy, Informal]
\label{thm:denoising-projection}
Under the linear interpolant $X_t = tX_1 + (1-t)X_0$, fix a condition $c$ and time $t\in(0,1)$. Let $\mathcal{P}_n$ denote the space of vector-valued polynomials in $X_t$ of degree at most $n$. The optimal degree-$n$ polynomial denoiser
\begin{equation}
\label{eq:polynomial-denoiser-projection}
v_n^\star = \arg\min_{v\in\mathcal{P}_n} \mathbb{E}\left[ \|v(X_t, t, c) - (X_1 - X_0)\|^2 \mid c \right],
\end{equation}
is exactly the orthogonal projection of the unconstrained flow-matching velocity $v^\star$ onto $\mathcal{P}_n$. 

Crucially, $v_n^\star$ satisfies the normal equations, meaning it perfectly matches the moment tensors of the target velocity up to order $n$:
\begin{equation}
\label{eq:flow-moment-tensors}
\mathbb{E}\left[ v_n^\star(X_t, t, c) \otimes X_t^{\otimes k} \mid c \right]
=
\mathbb{E}\left[ (X_1 - X_0) \otimes X_t^{\otimes k} \mid c \right]
:= G_k^t(P\mid c), 
\quad \forall k=0,\dots,n.
\end{equation}
Writing the data moments as $M_j(P\mid c) = \mathbb{E}[X_1^{\otimes j}\mid c]$, these matched tensors $G_k^t$ obey a triangular recurrence relation:
\begin{equation}
\label{eq:triangular-moment-relation}
G_k^t(P\mid c) = t^k M_{k+1}(P\mid c) + F_k^t(M_0,\dots,M_k),
\end{equation}
where $F_k^t$ is a deterministic function depending solely on $t$, the Gaussian moments of $X_0$, and the lower-order data moments $\{M_j\}_{j=0}^k$. Consequently, optimizing the degree-$n$ denoising objective uniquely determines the conditional data moments up to order $n+1$.
\end{theorem}

This theorem reveals that diffusion models inherently learn a hierarchy of amortized statistics, where the coefficients of the degree-zero, affine, and higher-degree terms correspond to the mean, covariance, and higher-order tensors, respectively. Since matching high-order tensors is computationally prohibitive in high-dimensional spaces, we isolate the affine case. This provides a tractable yet effective second-order instantiation, formalized as follows.

\begin{corollary}[Affine Projection Identifies Conditional Covariance]
\label{cor:affine-covariance}
Under the setting of Theorem~\ref{thm:denoising-projection}, assume $X_1\mid c$ has a finite second moment. Let $\mu(c)=\mathbb{E}[X_1\mid c]$ and $\Sigma(c)=\operatorname{Cov}(X_1\mid c)$. Then the degree-zero projection is exactly the conditional mean:
\begin{equation}
\label{eq:degree-zero-projection}
v_0^\star(z,t,c)=\mu(c),
\end{equation}
and the degree-one projection admits the affine form:
\begin{equation}
\label{eq:degree-one-affine-projection}
v_1^\star(z,t,c)
=
\mu(c)+R_t(c)\bigl(z-t\mu(c)\bigr),
\end{equation}
where the affine operator is given by
\begin{equation}
\label{eq:affine-covariance-operator}
R_t(c)
=
\bigl[t\Sigma(c)-(1-t)I\bigr]
\bigl[t^2\Sigma(c)+(1-t)^2I\bigr]^{-1}.
\end{equation}
For a fixed $t\in(0,1)$, the mapping $\Sigma(c)\mapsto R_t(c)$ is rigorously injective. Therefore, the affine component of the denoising velocity serves as an exact, injective transformation of the covariance.
\end{corollary}

\section{The Amortized Fr\'{e}chet Distance Loss}
\label{sec:amfd}

Motivated by the insight that affine denoising projections uniquely encode conditional covariances, we propose the Amortized Fr\'{e}chet Distance (AMFD) loss. We instantiate this matching process within a pre-trained representation space, as adopted by \citet{yang2026representation}. Below we detail the parameterization, followed by the full alternating training algorithm.

\subsection{Amortizer Parameterization}
\label{sec:amortizer-param}

Let $\{\phi_\ell\}_{\ell=1}^L$ be a set of frozen representation encoders. For a data branch $b\in\{r,g\}$, let $X_r^\ell=\phi_\ell(X_r)$ and $X_g^\ell=\phi_\ell(G_\theta(z,c))$ denote the real and generated features in $\mathbb{R}^{D_\ell}$. To capture their statistics, AMFD employs a neural amortizer that predicts a conditional mean $\mu_b^\ell(c)\in\mathbb{R}^{D_\ell}$ and an affine operator, which we parameterize residually as $R_{b,t}^\ell(c) = -I + tA_{b,t}^\ell(c)$ for numerical stability.

For a target feature $X_1^\ell \in \{X_r^\ell, X_g^\ell\}$ and noise $X_0^\ell\sim\mathcal{N}(0,I)$, we construct the interpolant $X_t^\ell=tX_1^\ell+(1-t)X_0^\ell$. Substituting the residual parameterization, our amortized affine denoiser takes the form:
\begin{equation}
\label{eq:amfd-affine-denoiser}
\hat v_b^\ell(X_t^\ell,t,c)
=
\mu_b^\ell(c)
+
R_{b,t}^\ell(c)\tilde X_t^\ell
=
\mu_b^\ell(c)
-
\tilde X_t^\ell
+
tA_{b,t}^\ell(c)\tilde X_t^\ell,
\end{equation}
where $\tilde X_t^\ell = X_t^\ell-t\mu_b^\ell(c)$ is the time-centered feature.

Materializing $A_{b,t}^\ell(c)$ explicitly would incur a prohibitive $O(D_\ell^2)$ memory cost in high-dimensional representations. We circumvent this by implicitly defining the operator via its JVP action on arbitrary vectors $u\in\mathbb{R}^{D_\ell}$. Given a differentiable neural mapping $f_{b,\psi}^\ell(s,t,c):\mathbb{R}^{D_\ell}\to\mathbb{R}^{D_\ell}$ governed by an auxiliary input $s$, the exact operator action is efficiently computed as:
\begin{equation}
\label{eq:jvp-operator}
A_{b,t}^\ell(c)u
=
\left.
\frac{\partial f_{b,\psi}^\ell(s,t,c)}{\partial s}
\right|_{s=0}
u
=
\operatorname{JVP}_{s}
\left[
f_{b,\psi}^\ell(s,t,c)
\right]_{s=0}(u).
\end{equation}
Detailed configurations for the JVP operation are provided in Appendix~\ref{app:details_jvp} and~\ref{app:general-amortizer-configs}.

\subsection{Training Algorithm}
\label{subsec:training_algo}

Figure~\ref{fig:method} (b) outlines the AMFD training procedure, which relies on an alternating optimization scheme (detailed alogrithm see Appendix~\ref{app:amfd_algorithm}). The amortizers learn to accurately estimate the conditional real and generated statistics, while the generator is trained to minimize the discrepancy between them. To ensure unbiased scale aggregation across diverse representation spaces and prevent high-dimensional features from disproportionately dominating the objective, all per-encoder mean, covariance, and amortizer losses employ the normalized
inner product $\langle a,b\rangle_\ell=\frac{1}{D_\ell}\left\langle a,b\right\rangle$ and its induced norm $\|a\|_\ell^2=\langle a,a\rangle_\ell$.

\textbf{Amortizer update.}
For branch \(b\in\{r,g\}\), define the target feature as
\[
X_1^\ell=
\begin{cases}
h_r^\ell, & b=r,\\
\operatorname{sg}[h_g^\ell], & b=g,
\end{cases}
\]
where $\operatorname{sg}[\cdot]$ denotes the stop-gradient operation. Given a sampled time $t\sim\mathcal{U}(0,1)$ and noise \(X_0^\ell\sim\mathcal N(0,I)\), we form the linear interpolant \(X_t^\ell\) and
train the affine denoiser (Eq.~\ref{eq:amfd-affine-denoiser}) via:
\begin{equation}
\label{eq:branch-amortizer-loss}
\mathcal L_{b}^{\ell}(\psi)
=
\mathbb E
\left[
\left\|
\hat v_b^\ell(X_t^\ell,t,c)
-
(X_1^\ell-X_0^\ell)
\right\|_\ell^2
\right]
+
\mathbb E
\left[
\left\|
\mu_b^\ell(c)-X_1^\ell
\right\|_\ell^2
\right].
\end{equation}
At the population optimum, the diffusion matching term natively identifies the conditional mean and covariance functions. The explicit mean-regression term is theoretically redundant but practically serves as a regularization factor that consistently improves empirical convergence. The total amortizer objective sums over all branches and encoders:
\begin{equation}
\label{eq:amortizer-loss}
\mathcal L_\mathrm{amort}(\psi)
=
\sum_{\ell=1}^L
\left(
\mathcal L_{r}^{\ell}(\psi)
+
\mathcal L_{g}^{\ell}(\psi)
\right).
\end{equation}

\textbf{Generator mean loss.}
The generator should modify its features along the direction that explicitly minimizes the mean discrepancy \(\mu_g^\ell(c)-\mu_r^\ell(c)\). We optimize this using the first-variation objective:
\begin{equation}
\label{eq:generator-mean-loss}
\mathcal L_{\mu}^{\ell}(\theta)
=
\left\langle
\operatorname{sg}\!\left[
\mu_g^\ell(c)-\mu_r^\ell(c)
\right],
h_g^\ell
\right\rangle_\ell .
\end{equation}
Crucially, gradients are only routed through the generated features \(h_g^\ell\), while the amortizer-predicted means act as frozen empirical targets.

\textbf{Generator covariance loss.}
For a given encoder $\ell$, we define the centered generated feature as $u^\ell = h_g^\ell-\operatorname{sg}[\mu_g^\ell(c)]$. Ideally, the generator would directly minimize the quadratic discrepancy between the two conditional covariance matrices. Though the amortizer-predicted affine operator $A_t^\ell(c)$ maps uniquely to $\Sigma^\ell(c)$, directly substituting $A_t^\ell(c)$ for $\Sigma^\ell(c)$ is poorly conditioned.

Instead, we project the discrepancy through the local Jacobian of the mapping $A_t^\ell \mapsto \Sigma^\ell$ evaluated at the generated branch. This yields the Jacobi operator $\mathcal{J}_{g,t}^\ell(c)u = (1+t)u-t^2A_{g,t}^\ell(c)u$, which defines the local pullback force:
\begin{equation}
\label{eq:jacobi-cov-force}
\delta a_{\mathrm J}^\ell = \mathcal{J}_{g,t}^\ell(c) \left( A_{g,t}^\ell(c)-A_{r,t}^\ell(c) \right) \mathcal{J}_{g,t}^\ell(c)u^\ell.
\end{equation}
We then compute the generator covariance loss as:
\begin{equation}
\label{eq:generator-cov-loss-jacobi}
\mathcal{L}_{\mathrm{cov}}^{\ell}(\theta) = \omega_{\mathrm{cov}}(t) \left\langle u^\ell, \delta a_{\mathrm J}^\ell \right\rangle_\ell,
\end{equation}
where $\omega_{\mathrm{cov}}(t)$ is a scaling coefficient (derivation detailed in Appendix~\ref{app:jacobi-covariance}) required to approximately match the weighting of the FD objective (Eq.~\ref{eq:frechet-distance}). During the generator update, the amortizers are frozen, but the internal JVP operator actions (Eq.~\ref{eq:jvp-operator}) are dynamically tethered to $u^\ell$, ensuring exact gradient flow to the generator. Notably, ignoring this local geometry yields a non-Jacobi variant $\delta a_{\mathrm{noJ}}^\ell = (A_{g,t}^\ell(c) - A_{r,t}^\ell(c))u^\ell$, which we empirically found to be significantly less effective, thereby validating our Jacobi-adjusted formulation.

\textbf{Multi-representation normalization.}
Aggregating losses across multiple heterogeneous encoders often produces generator gradients with highly disproportionate magnitudes. Similar to FD-loss~\citep{yang2026representation}, we normalize each loss via a per-sample gradient proxy, thereby stabilizing the optimization process. For encoder \(\ell\), the proxy vector is defined as:
\begin{equation}
\label{eq:generator-gradient-proxy}
g_{\mathrm{proxy}}^\ell
=
\operatorname{sg}\!\left[
\mu_g^\ell(c)-\mu_r^\ell(c)
\right]
+
2\omega_{\mathrm{cov}}(t)\delta a_{\mathrm J}^\ell .
\end{equation}
The normalized per-encoder generator loss is subsequently computed as:
\begin{equation}
\label{eq:normalized-generator-loss}
\mathcal L_{\mathrm{norm}}^\ell(\theta)
=
\mathbb E
\left[
\frac{
\mathcal L_{\mu}^{\ell}(\theta)
+
\mathcal L_{\mathrm{cov}}^{\ell}(\theta)
}{
\left(
\sum_{j=1}^{D_\ell}
\left(g_{\mathrm{proxy},j}^\ell\right)^2
\right)^{\rho}
+
\epsilon
}
\right],
\end{equation}
where \(\rho\) and \(\epsilon\) serve as scaling hyperparameters. The final unified generator objective aggregates these normalized losses across all \(L\) representation encoders:
\begin{equation}
\label{eq:generator-loss}
\mathcal L_\mathrm{gen}(\theta)
=
\sum_{\ell=1}^L
\mathcal L_{\mathrm{norm}}^\ell(\theta) .
\end{equation}

\section{Experiments}
\label{sec:experiments}

In this section, we empirically evaluate our proposed AMFD framework. First, we isolate and ablate the specific design choices of the amortizers. Second, we assess AMFD in ImageNet post-training and extend our study to native generative spaces. Finally, we investigate the transferability of the AMFD objective to text-to-image post-training across both pixel and latent spaces.

\subsection{Amortizer Implementation}
We instantiate one amortizer for each representation encoder, yielding a total of $L$ amortizers with randomly initialized weights. For global feature representations, the amortizer is an AdaLN-conditioned MLP~\citep{li2024autoregressive}; in text-to-image models, a cross-attention module is integrated into every AdaLN block to process text conditioning. For native generative spaces, the amortizer mirrors the generator's architecture. We refer to this condition-aware configuration as AMFD-C (Conditional). For unconditional generation, denoted as AMFD-U (Unconditional), we retain the same network architecture but replace the conditioning input with a fixed dummy condition. The real and fake branches share identical amortizer parameters, relying exclusively on the time step $t$ to distinguish the specific prediction task. More details are provided in Appendix~\ref{app:details_ablation} and~\ref{app:general-amortizer-configs}.

\subsection{Amortizer Ablations}
\label{sec:amortizer-ablation}

\begin{figure}[t!]
    \centering
    \begin{minipage}[c]{0.31\linewidth}
        \setlength{\abovecaptionskip}{1pt}
        \centering
        \includegraphics[width=\linewidth]{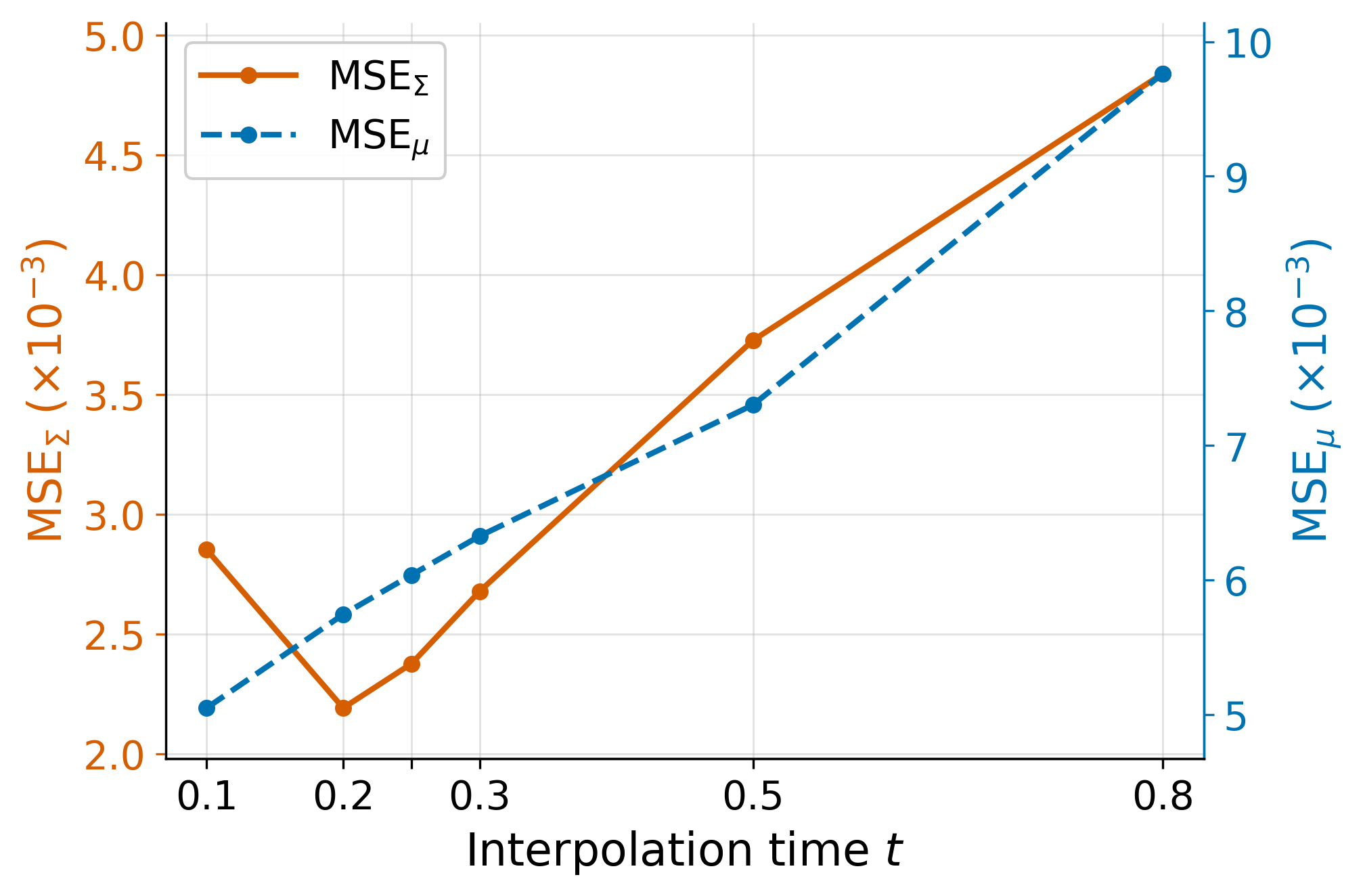}
        \caption{Estimation accuracy across interpolation time $t$.}
        \label{fig:t_sweep}
    \end{minipage}
    \hfill
    \begin{minipage}[c]{0.31\linewidth}
        \setlength{\abovecaptionskip}{1pt}
        \centering
        \includegraphics[width=\linewidth]{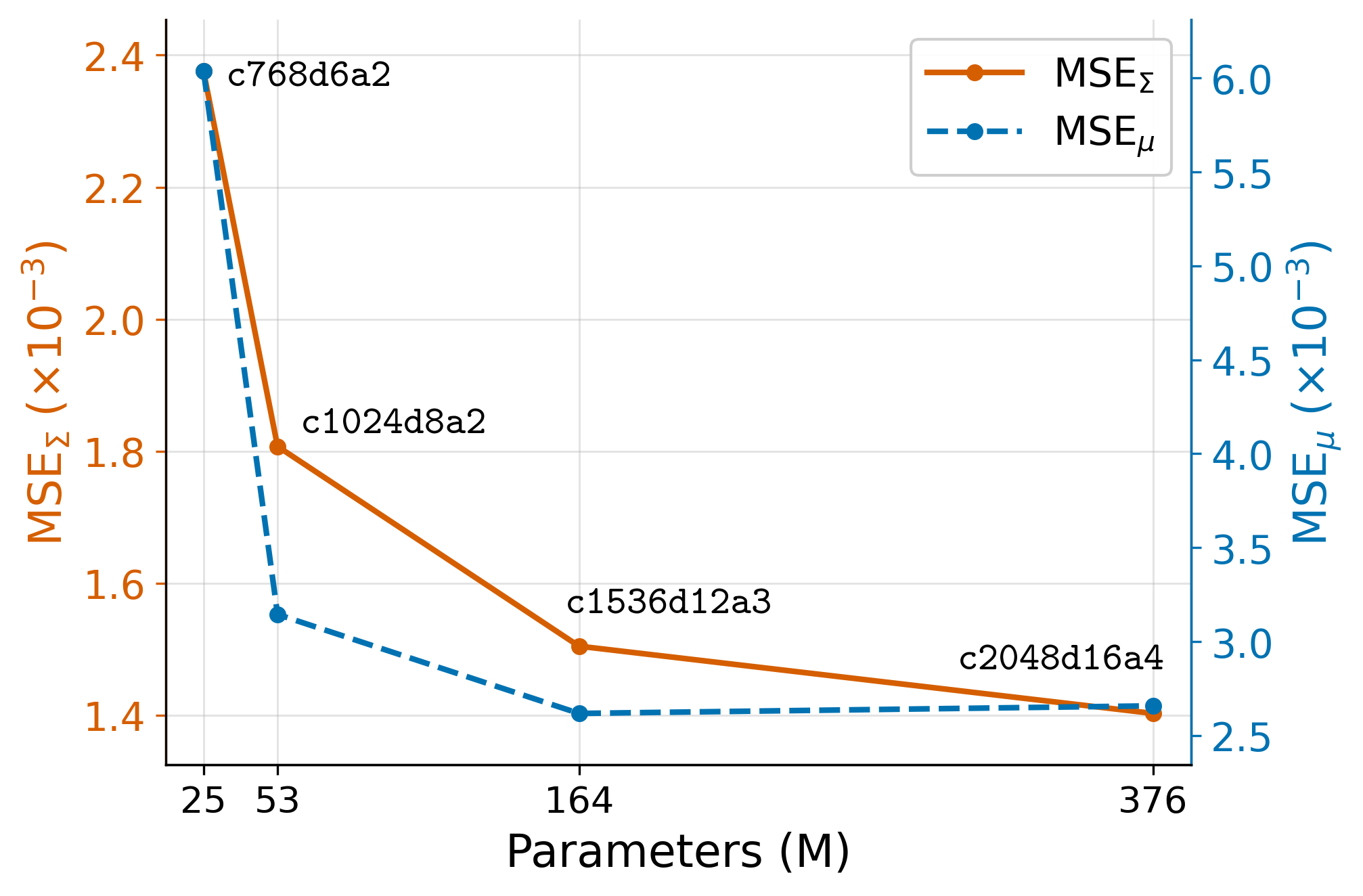}
        \caption{Estimation accuracy across amortizer MLP sizes.}
        \label{fig:model_size}
    \end{minipage}
    \hfill
    \begin{minipage}[c]{0.31\linewidth}
        \centering
        \small
        \setlength{\tabcolsep}{2pt}
        \begin{tabular}{lcc}
            \toprule
            Variant & MSE$_\mu$ & MSE$_\Sigma$ \\
            \midrule
            AMFD-C & 5.75 & 2.18 \\
            AMFD-U & 0.0106 & 0.114 \\
            \bottomrule
        \end{tabular}
        \captionof{table}{Estimation accuracy with (AMFD-C) and without (AMFD-U) class conditioning (MSE $\times 10^{-3}$).}
        \label{tab:condition}
    \end{minipage}
\end{figure}

To isolate estimation accuracy, we train the amortizer solely on real Inception-v3~\citep{szegedy2016rethinking} features. We evaluate its performance using MSE$_\mu$ and MSE$_\Sigma$ to measure the discrepancy against the analytical ground truth (details in Appendix~\ref{app:details_ablation}).

\textbf{Effect of interpolation time \(t\).}
Unlike standard diffusion training that learns over the entire time continuum, Eq.~\ref{eq:affine-covariance-operator} shows that the target moments can be perfectly identified at any fixed \(t\in(0,1)\). By sweeping \(t\) across this interval, we observe that smaller \(t\) values yield consistently lower amortization errors for both the mean and covariance (Figure~\ref{fig:t_sweep}).

\textbf{Scaling model capacity.}
We evaluate estimation accuracy across varying MLP capacities. As depicted in Figure~\ref{fig:model_size}, increasing model size steadily reduces approximation error. However, as $\mathrm{MSE}_\mu$ plateaus and $\mathrm{MSE}_\Sigma$ gains diminish at the \texttt{c2048d16a4} scale, we adopt it as our default.

\textbf{Effect of class conditioning.} As shown in Table~\ref{tab:condition}, the unconditional setting (AMFD-U) achieves considerably higher estimation accuracy than its conditional counterpart (AMFD-C). This gap likely stems from the inherent simplicity of modeling a single global marginal distribution, compounded by the sparse per-class data density that limits statistical efficiency during conditional matching.

\subsection{ImageNet Post-Training}
\label{sec:imagenet-post-training}
\begin{table}[t]
\centering
\caption{System-level comparison for class-conditional generation on ImageNet 256$\times$256. $\dagger$ indicates using CFG interval. \protect\colorbox{orange!30}{FD-loss baseline} and \protect\colorbox{green!40}{AMFD} employ SIM (SigLIP, Inception, MAE).}
\label{tab:system-level-amfd}
\scriptsize
\begin{tabular}{llccccccc}
\toprule
Method & Space & NFE & Params & FDr\(^6\)\(\downarrow\) & FID\(\downarrow\) & IS\(\uparrow\) & Precision\(\uparrow\) & Recall\(\uparrow\) \\
\midrule

\color{gray!70} 50k validation images
& \color{gray!70} - & \color{gray!70} - & \color{gray!70} - & \color{gray!70} 1.00 & \color{gray!70} 1.68 & \color{gray!70} 232.2 & \color{gray!70} 0.75 & \color{gray!70} 0.66 \\

\midrule
\textit{Multi-step models} \\
VAR-d30~\citep{tian2024visual}
& discrete & \(10\times2\) & 2B & 6.70 & 1.97 & 304.6 & 0.82 & 0.59 \\
BAR-L~\citep{yu2026autoregressive}
& discrete & \(256\times 2 \text{\tiny $\times 4$}\) & 1.1B & 3.57 & 1.01 & 281.9 & 0.77 & 0.68 \\

SiT-XL/2~\citep{ma2024sit}
& latent & \(250\times 2\) & 675M & 8.44 & 2.12 & 256.7 & 0.81 & 0.60 \\
MAR-L~\citep{li2024autoregressive}
& latent & \(256\times 2 \text{\tiny $\times 100$}\) & 478M & 6.68 & 1.80 & 293.4 & 0.80 & 0.60 \\
FlowAR-H~\citep{ren2024flowar}
& latent & \(50\times 2^\dagger\) & 1.9B & 6.13 & 1.68 & 274.1 & 0.80 & 0.62 \\
MAR-H~\citep{li2024autoregressive}
& latent & \(256\times 2 \text{\tiny $\times 100$}\) & 942M & 5.61 & 1.56 & 299.5 & 0.80 & 0.62 \\
MAR-L, DeTok~\citep{yang2025latent}
& latent & \(256\times 2 \text{\tiny $\times 100$}\) & 478M & 5.49 & 1.39 & 306.2 & 0.81 & 0.62 \\
REG~\citep{wu2026representation}
& latent & \(250\times 2^\dagger\) & 685M & 4.64 & 1.54 & 302.9 & 0.78 & 0.62 \\
SiT-XL/2-REPA~\citep{yu2024representation}
& latent & \(250\times 2^\dagger\) & 675M & 5.45 & 1.42 & 306.1 & 0.80 & 0.65 \\
LightningDiT~\citep{yao2025reconstruction}
& latent & \(250\times 2\) & 675M & 4.57 & 1.42 & 294.3 & 0.80 & 0.64 \\
DDT-XL~\citep{wang2026ddt}
& latent & \(250\times 2\) & 675M & 5.70 & 1.26 & 309.3 & 0.79 & 0.66 \\
REPA-E~\citep{leng2025repa}
& latent & \(250\times 2^\dagger\) & 676M & 3.04 & 1.17 & 298.3 & 0.79 & 0.66 \\
RAE-XL~\citep{zheng2025diffusion}
& latent & \(50\times 2^\dagger\) & 839M & 3.26 & 1.16 & 261.0 & 0.77 & 0.67 \\
PixNerd-XL~\citep{wang2025pixnerd}
& pixel & \(100\times 2\) & 1.0B & 5.01 & 2.10 & 318.8 & 0.81 & 0.59 \\
JiT-B~\citep{li2026back}
& pixel & \(50\times 2\times 2^\dagger\) & 131M & 15.65 & 3.71 & 269.0 & 0.81 & 0.50 \\
+ FD-loss~\citep{yang2026representation}
& pixel & 1 & 131M & \cellcolor{orange!30}5.53 & 1.00 & 344.6 & 0.78 & 0.60 \\
\color{gray!70} + AMFD-C (Ours) & \color{gray!70} pixel & \color{gray!70} 1 & \color{gray!70} 131M & \color{gray!70} 4.75 & \color{gray!70} 0.95 & \color{gray!70} 319.4 & \color{gray!70} 0.77 & \color{gray!70} 0.67 \\
+ AMFD-U (Ours)
& pixel & 1 & 131M & \cellcolor{green!40}3.91 & 0.95 & 325.2 & 0.77 & 0.60 \\
JiT-L~\citep{li2026back}
& pixel & \(50\times 2\times 2^\dagger\) & 459M & 10.73 & 2.59 & 288.5 & 0.79 & 0.59 \\
+ FD-loss~\citep{yang2026representation}
& pixel & 1 & 459M & \cellcolor{orange!30}3.24 & 0.77 & 317.3 & 0.77 & 0.66 \\
\color{gray!70} + AMFD-C (Ours) & \color{gray!70} pixel & \color{gray!70} 1 & \color{gray!70} 459M & \color{gray!70} 2.67 & \color{gray!70} 0.87 & \color{gray!70} 325.1 & \color{gray!70} 0.76 & \color{gray!70} 0.69 \\
+ AMFD-U (Ours)
& pixel & 1 & 459M & \cellcolor{green!40}2.02 & 0.85 & 319.8 & 0.77 & 0.65 \\
JiT-H~\citep{li2026back}
& pixel & \(50\times 2\times 2^\dagger\) & 953M & 7.66 & 1.97 & 296.0 & 0.78 & 0.63 \\
+ FD-loss~\citep{yang2026representation}
& pixel & 1 & 953M & \cellcolor{orange!30}2.65 & 0.75 & 313.0 & 0.76 & 0.66 \\
\color{gray!70} + AMFD-C (Ours) & \color{gray!70} pixel & \color{gray!70} 1 & \color{gray!70} 953M & \color{gray!70} 2.15 & \color{gray!70} 0.85 & \color{gray!70} 328.1 & \color{gray!70} 0.76 & \color{gray!70} 0.69 \\
+ AMFD-U (Ours)
& pixel & 1 & 953M & \cellcolor{green!40}1.79 & 0.83 & 312.2 & 0.76 & 0.67 \\

\midrule
\textit{One-step models} \\
Drift-L~\citep{deng2026generative}
& latent & 1 & 463M & 10.92 & 1.53 & 257.2 & 0.79 & 0.63 \\
Drift-L~\citep{deng2026generative}
& pixel & 1 & 465M & 10.51 & 1.43 & 305.8 & 0.81 & 0.60 \\
iMF-XL~\citep{geng2026improved}
& latent & 1 & 610M & 8.39 & 1.82 & 278.9 & 0.78 & 0.63 \\
+ FD-loss~\citep{yang2026representation}
& latent & 1 & 610M & 2.45 & 0.76 & 301.3 & 0.77 & 0.67 \\

pMF-B~\citep{lu2026one}
& pixel & 1 & 118M & 13.70 & 3.31 & 254.6 & 0.81 & 0.52 \\
+ FD~\citep{yang2026representation}
& pixel & 1 & 118M & \cellcolor{orange!30}3.50 & 0.85 & 331.4 & 0.77 & 0.65 \\
\color{gray!70} + AMFD-C (Ours) & \color{gray!70} pixel & \color{gray!70} 1 & \color{gray!70} 118M & \color{gray!70} 3.94 & \color{gray!70} 0.95 & \color{gray!70} 315.9 & \color{gray!70} 0.77 & \color{gray!70} 0.66 \\
+ AMFD-U (Ours)
& pixel & 1 & 118M & \cellcolor{green!40}3.43 & 0.92 & 310.1 & 0.77 & 0.65 \\
pMF-L~\citep{lu2026one}
& pixel & 1 & 410M & 9.09 & 2.72 & 261.7 & 0.81 & 0.56 \\
+ FD-loss~\citep{yang2026representation}
& pixel & 1 & 410M & \cellcolor{orange!30}2.09 & 0.78 & 309.2 & 0.76 & 0.67 \\
\color{gray!70} + AMFD-C (Ours) & \color{gray!70} pixel & \color{gray!70} 1 & \color{gray!70} 410M & \color{gray!70} 2.25 & \color{gray!70} 0.88 & \color{gray!70} 321.8 & \color{gray!70} 0.76 & \color{gray!70} 0.68 \\
+ AMFD-U (Ours)
& pixel & 1 & 410M & \cellcolor{green!40}2.01 & 0.86 & 306.7 & 0.76 & 0.68 \\
pMF-H~\citep{lu2026one}
& pixel & 1 & 935M & 6.87 & 2.29 & 267.2 & 0.80 & 0.59 \\
+ FD-loss~\citep{yang2026representation}
& pixel & 1 & 935M & \cellcolor{orange!30}1.89 & 0.77 & 310.1 & 0.77 & 0.68 \\
\color{gray!70} + AMFD-C (Ours) & \color{gray!70} pixel & \color{gray!70} 1 & \color{gray!70} 935M & \color{gray!70} 1.93 & \color{gray!70} 0.86 & \color{gray!70} 323.0 & \color{gray!70} 0.76 & \color{gray!70} 0.69 \\
+ AMFD-U (Ours)
& pixel & 1 & 935M & \cellcolor{green!30}1.75 & 0.85 & 307.3 & 0.76 & 0.68 \\

\bottomrule
\end{tabular}
\end{table}

\begin{figure}[t!]
    \centering
    \begin{minipage}[c]{0.34\linewidth}
        \setlength{\abovecaptionskip}{1pt}
        \centering
        \includegraphics[width=\linewidth]{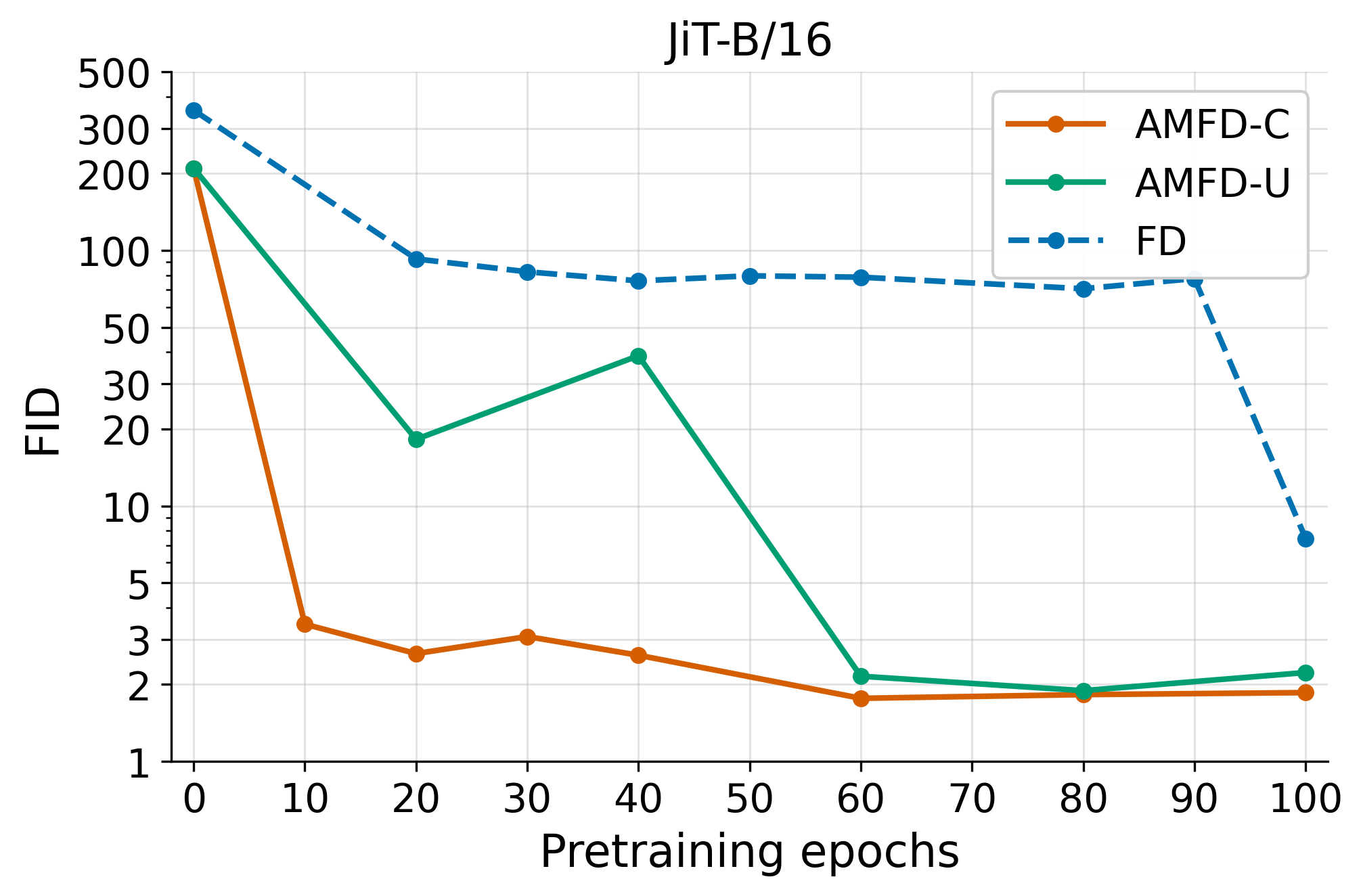}
        \caption{Impact of pretraining epochs on JiT-B/16 post-training.}
        \label{fig:jitb_pretrain}
    \end{minipage}
    \hfill
    \begin{minipage}[c]{0.34\linewidth}
        \setlength{\abovecaptionskip}{1pt}
        \centering
        \includegraphics[width=\linewidth]{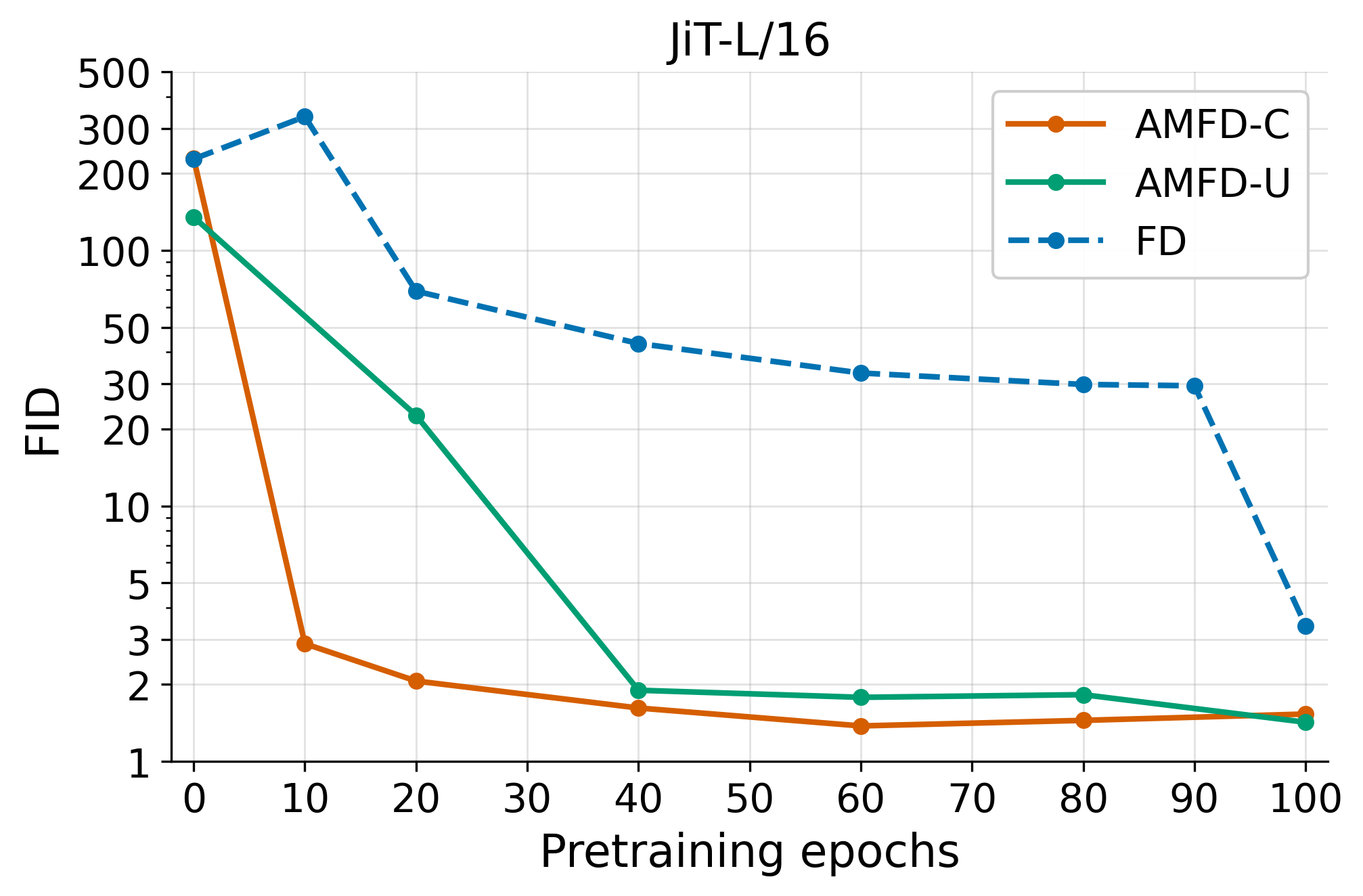}
        \caption{Impact of pretraining epochs on JiT-L/16 post-training.}
        \label{fig:jitl_pretrain}
    \end{minipage}
    \hfill
    \begin{minipage}[c]{0.26\linewidth}
        \setlength{\abovecaptionskip}{2pt}
        \centering
        \includegraphics[width=\linewidth]{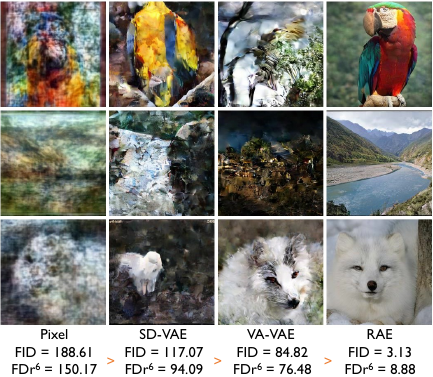}
        \caption{AMFD in native generative spaces.}
        \label{fig:native_space}
    \end{minipage}
\end{figure}
\textbf{Superior dynamics and results when matching in global representation spaces.} 
We evaluate AMFD for ImageNet post-training using both the multi-step JiT and one-step pMF models. As shown in Table~\ref{tab:system-level-amfd}, AMFD-U substantially improves the FDr$^6$ score over the FD-loss baseline across both architectures (see Appendix~\ref{app:add_metrics} for full FDr metrics and Appendix~\ref{app:visualization} for qualitative samples). The table also reveals that AMFD-C slightly underperforms AMFD-U in final generation quality. We attribute this to two intertwined factors. First, amortizing unconditional marginal moments is an inherently easier optimization task than fitting fine-grained conditional moments, a difference in estimation accuracy explicitly corroborated in Table~\ref{tab:condition}. Second, the discrete class labels in ImageNet represent a relatively simple semantic condition; in such low-complexity conditional spaces, the marginal benefits of conditional matching are outweighed by the increased estimation difficulty. We further study the robustness of AMFD against FD-loss when the base model is far from a one-step generative regime by analyzing post-training convergence across varying pre-training checkpoints. Notably, while FD-loss completely diverges on early-stage checkpoints, both AMFD variants consistently converge to competitive FIDs (Figures~\ref{fig:jitb_pretrain} and \ref{fig:jitl_pretrain}). This demonstrates that the primary driver of robustness is the neural formulation itself. Furthermore, AMFD-C achieves strong FID scores with even fewer pre-training epochs than AMFD-U. This indicates that, despite being harder to fit perfectly, precise class-aware supervision provides further training stabilization.

\textbf{Exploration in native generative space.}
Empowered by the space complexity efficiency of AMFD, we further investigate whether matching solely the first and second moments within native generative spaces is sufficient to capture the target data distribution. Our evaluation spans four distinct environments with increasing semantic density: raw pixel space, SD-VAE~\citep{rombach2022high} latent space, VA-VAE~\citep{yao2025reconstruction} latent space, and RAE~\citep{zheng2025diffusion} latent space. As detailed in Figure~\ref{fig:native_space}, both qualitative visualizations and quantitative metrics (FID and FDr$^6$) reveal that generative performance improves progressively as the semantic strength of the space increases. Notably, while the lower-semantic spaces struggle to produce coherent structures, the RAE latent space successfully yields reasonable visual quality with a competitive FID of $3.13$. This contrast suggests that matching first- and second-order statistics becomes progressively more effective as the representation space encodes richer and more structured semantics.

\subsection{Text-to-Image Post-Training}
\label{sec:t2i-post-training}

\begin{table}[t]
\centering
\caption{GenEval and PickScore for text-to-image post-training. 
Higher is better for all metrics.}
\label{tab:text-to-image}
\setcounter{rownum}{0} 
\scriptsize
\setlength{\tabcolsep}{3.3pt}
\begin{tabular}{llccccccccc}
\toprule
Row
& Method 
& NFE
& Single Obj. 
& Two Obj. 
& Counting 
& Colors 
& Position 
& Color Attr. 
& Overall 
& PickScore \\
\midrule
& \multicolumn{2}{l}{\textit{Pixel space, multi-step}} \\
\rowidx & PixelGen & 25$\times$2$\times$2
& 0.994 & 0.886 & 0.578 & 0.902 & 0.708 & 0.698 & \textbf{0.794} & \textbf{21.02} \\
\rowidx & FD-SIM~\citep{yang2026representation} & 1
& 0.984 & 0.795 & 0.416 & 0.753 & 0.535 & 0.415 & 0.650 & 20.33 \\
\rowidx & AMFD-U-SIM (Ours) & 1
& 0.975 & 0.831 & 0.469 & 0.769 & 0.658 & 0.578 & 0.713 & 20.68 \\
\rowidx & AMFD-C-SIM (Ours) & 1
& 0.969 & 0.864 & 0.513 & 0.846 & 0.723 & 0.645 & \underline{0.760} & \underline{20.80} \\
\midrule
& \multicolumn{2}{l}{\textit{VAE latent space, multi-step}} \\
\rowidx & FLUX.2 [klein] 4B Base & 50$\times$2
& 1.000 & 0.864 & 0.759 & 0.899 & 0.505 & 0.642 & 0.778 & \textbf{21.50} \\
\rowidx & FD-SIM~\citep{yang2026representation} & 1
& 1.000 & 0.934 & 0.653 & 0.824 & 0.448 & 0.598 & 0.743 & 21.31 \\
\rowidx & AMFD-U-SIM (Ours) & 1
& 0.997 & 0.917 & 0.644 & 0.827 & 0.523 & 0.590 & 0.749 & 21.37 \\
\rowidx & AMFD-C-SIM (Ours) & 1
& 1.000 & 0.957 & 0.750 & 0.910 & 0.733 & 0.650 & \underline{0.833} & 21.41 \\
\rowidx & AMFD-C-10 enc. (Ours) & 1
& 1.000 & 0.962 & 0.766 & 0.886 & 0.720 & 0.685 & \textbf{0.836} & \underline{21.46} \\
\midrule
& \multicolumn{2}{l}{\textit{VAE latent space, few-step}} \\
\rowidx & FLUX.2 [klein] 4B & 4
& 0.994 & 0.904 & 0.791 & 0.880 & 0.575 & 0.623 & 0.794 & \textbf{21.85} \\
\rowidx & DMD2~\citep{yin2025improved} & 1
& 0.997 & 0.894 & \textbf{0.806} & 0.864 & 0.603 & 0.660 & 0.804 & - \\
\rowidx & FD-SIM~\citep{yang2026representation} & 1
& \textbf{1.000} & 0.944 & 0.716 & 0.878 & 0.618 & 0.658 & 0.802 & 21.62 \\
\rowidx & iRDM~\citep{feng2026representation} & 1
& 0.994 & 0.924 & 0.756 & \textbf{0.923} & 0.650 & 0.708 & 0.826 & \underline{21.82} \\
\rowidx & AMFD-U-SIM (Ours) & 1
& 0.994 & 0.929 & 0.741 & 0.902 & 0.638 & 0.678 & 0.813 & 21.77 \\
\rowidx & AMFD-C-SIM (Ours) & 1
& 0.997 & 0.955 & 0.791 & \textbf{0.923} & 0.670 & 0.740 & \textbf{0.846} & \textbf{21.85} \\
\rowidx & AMFD-C-10 enc. (Ours) & 1
& \textbf{1.000} & \textbf{0.957} & 0.778 & 0.920 & \textbf{0.680} & \textbf{0.733} & \underline{0.845} & \underline{21.82} \\
\bottomrule
\end{tabular}
\end{table}

We evaluate AMFD on text-to-image generation at 512$\times$512 resolution using the GenEval~\citep{ghosh2023geneval} (instruction following) and PickScore~\citep{kirstain2023pick} (human preference) benchmarks across diverse architectures: the pixel-space PixelGen~\citep{ma2026pixelgen}, the latent-space FLUX.2 [klein] 4B Base~\citep{blackforestlabs2026flux2klein}, and its 4-step distilled counterpart. As shown in Table~\ref{tab:text-to-image}, our AMFD establishes a new paradigm with several pivotal insights:

\textbf{Inherent stability in complex T2I.} Even without conditioning, AMFD-U-SIM consistently outperforms the FD-SIM baseline across all architectures on both GenEval and PickScore (Rows 2 vs. 3, 6 vs. 7, and 12 vs. 14). This confirms that our amortized formulation inherently provides better training dynamics in highly complex text-to-image scenarios.

\textbf{The critical role of conditioning.} When post-training using a few-step teacher (which inherently operates closer to a one-step generator and already possesses robust conditional alignment), AMFD-U-SIM performs remarkably well (Row 14). However, when using a standard multi-step diffusion model, condition-aware matching becomes absolutely crucial. For instance, in the multi-step FLUX.2 Base, AMFD-C-SIM dramatically boosts the GenEval score from $0.749$ to $0.833$ compared to its unconditional counterpart (Rows 7 vs. 8), unlocking similarly massive gains in PixelGen ($0.713$ to $0.760$, Rows 3 vs. 4). This completely validates the necessity of condition-aware matching for standard diffusion distillation—a stark contrast to our ImageNet findings, driven by the vastly higher semantic complexity of text prompts compared to class labels.

\textbf{Sufficiency of the first two moments.} FD-loss~\citep{yang2026representation} has empirically established that matching the first two statistical moments yields strong generative performance. We further corroborate this by comparing our AMFD (Rows 15 \& 16) with the recent iRDM~\citep{feng2026representation} (Row 13), which theoretically matches \emph{all} statistical moments via an MMD criterion. By achieving superior GenEval scores and equal or better PickScores, our results further demonstrate that, when empowered by neural amortization, aligning merely the first two moments already provides abundant supervisory signals.

\textbf{Effectiveness of the compact SIM set.} Expanding the representation set from SIM to 10 encoders (Table~\ref{tab:repr_encoders}, similar to~\citet{feng2026representation}) yields comparable performance across both FLUX.2 [klein] models (Rows 8 vs. 9; 15 vs. 16). While standard metrics may not capture all real-world visual nuances, this suggests the lightweight SIM ensemble already provides sufficient features for strong text-to-image alignment, offering a practical balance of efficiency and quality.

Ultimately, these combined advantages allow our one-step AMFD models to significantly surpass both the multi-step FLUX.2 [klein] 4B Base teacher and its 4-step distilled variant on GenEval, while achieving highly competitive, on-par performance on PickScore (\textit{e.g.}, matching the $21.85$ PickScore of the 4-step teacher). This demonstrates the formidable capability of our method. Qualitative results are in Figure~\ref{fig:t2i-samples} and Appendix Figure~\ref{fig:add-t2i}.
\section{Limitations}
\label{sec:limitations}

While effective, our study presents several limitations. First, although our theory accommodates arbitrary-order moments, AMFD is restricted to matching only the first two moments, leaving higher-order statistics uncaptured. Second, our reliance on off-the-shelf representation encoders, which are typically pretrained at low resolutions, bottlenecks the direct applicability to high-resolution generation. Finally, AMFD serves as a post-training objective, while it is currently unable to train generative models entirely from scratch.

\section{Conclusion}
\label{sec:conclusion}
In this work, we introduced \emph{amortized moment matching}, a principled framework that leverages neural amortizers to provide scalable distributional training signals for generative models. By analyzing diffusion denoisers through the lens of polynomial projections, we established that they inherently learn a hierarchy of conditional statistics. We instantiated the tractable second-order case as the Amortized Fr\'{e}chet Distance (AMFD) loss, a post-training paradigm designed to elevate the synthesis quality of existing models and convert multi-step models into efficient one-step generators. By efficiently matching conditional means and covariance operators, AMFD yields superior one-step generation performance in ImageNet post-training. Furthermore, it scales seamlessly to complex text-to-image architectures across both pixel and VAE latent spaces, demonstrating strong instruction-following proficiencies and high-fidelity generation.


\bibliography{iclr2026_conference}
\bibliographystyle{iclr2026_conference}

\newpage
\appendix
\section*{Appendix}
\section{Related Work}
\label{sec:related}

\subsection{Diffusion and Flow Matching}

Diffusion models~\citep{sohl2015deep,song2019generative,ho2020denoising} and flow matching~\citep{liu2022flow,lipman2022flow} typically learn unrestricted denoising functions or velocity fields to iteratively transport samples from a prior to the data distribution. From a statistical perspective, classical Tweedie identities connect the optimal Gaussian denoiser to posterior moments conditioned on a particular noisy observation~\citep{manor2024posterior,boys2023tweedie}. Our perspective differs fundamentally: by restricting the denoising regression to polynomial function classes, we study its projection coefficients. Crucially, these coefficients depend on the conditional data distribution, rather than on an individual noised sample, and form a hierarchy that identifies increasingly high-order data moments. Accordingly, we do not use the learned denoiser for generative sampling. Instead, we utilize these amortized coefficients as explicit distributional training signals for a separate generator.

\subsection{Moment Matching in Representation Spaces}

Moment matching provides a direct approach to generative modeling by comparing statistics of real and generated distributions. Traditional methods rely on nonparametric frameworks, optimizing empirical metrics such as kernel mean embeddings and maximum mean discrepancy (MMD) from finite sample sets~\citep{sriperumbudur2010hilbert,gretton2012kernel,li2015generative,dziugaite2015training,li2017mmd}. A complementary line of work demonstrates that pretrained visual encoders provide highly effective spaces for generative supervision. Perceptual objectives~\citep{zhang2018unreasonable} compare paired images using frozen deep features, projected GANs~\citep{sauer2021projected} construct discriminators on top of pretrained representations, representation alignment methods~\citep{yu2024representation} supervise internal generative features, and representation encoders~\citep{zheng2025diffusion} train generative models natively in representation spaces. Recent work has begun to directly optimize discrepancies between real and generated distributions in representation spaces. Drifting Models~\citep{deng2026generative} construct a kernel-induced field that evolves the generated distribution during training. FD-loss~\citep{yang2026representation} demonstrates that representation-space means and covariances can be explicitly optimized by decoupling population-level moment estimation from generator minibatches. Several concurrent works have also explored representation-space distribution matching to enable strong one-step generation. Specifically, Representation Distribution Matching (RDM)~\citep{feng2026representation} utilizes MMD estimators, while Three-Body Scattering Modeling (TBSM)~\citep{sun2026three} employs an online tracker to learn an energy-distance descent field. Rather than estimating empirical marginal statistics or learning sample-dependent transport fields, AMFD amortizes condition-level moment coefficients derived from polynomial projections of a diffusion denoiser.

\subsection{Distillation and Post-Training for One-Step Generation}

State-of-the-art generative paradigms, such as diffusion models~\citep{sohl2015deep,song2019generative,ho2020denoising} and flow matching~\citep{liu2022flow,lipman2022flow}, achieve unprecedented synthesis quality but typically rely on slow, iterative generation processes. To address this, generative distillation and post-training techniques have been extensively developed to compress these iterative processes into efficient one-step generators. Existing approaches can be broadly categorized into several directions. Direct distillation methods~\citep{luhman2021knowledge,salimans2022progressive} explicitly matches the input-output or intermediate mappings of a multi-step teacher. Trajectory consistency methods~\citep{song2023consistency,boffi2024flow,liu2025learning,geng2025mean,zhou2025inductive} enforce consistency along the teacher's deterministic ODE trajectory. GAN-based post-training methods~\citep{sauer2024fast,lin2025diffusion} leverage adversarial objectives to directly align the one-step generated distribution with the real data distribution. Furthermore, variational score distillation (VSD)~\citep{wang2023prolificdreamer,yin2024one} optimize the one-step generator by estimating and matching the score functions of the target and generated distributions. Although AMFD shares the alternating optimization structure of adversarial and VSD-based post-training, its amortizers serve a fundamentally different statistical role. A GAN discriminator learns to classify real versus generated samples, and VSD-style networks estimate the full score fields of intermediate noisy distributions. In contrast, AMFD neural amortizers directly regress specific statistical targets: conditional means and covariance-dependent operators. Consequently, AMFD bypasses traditional adversarial dynamics and the necessity to reproduce a fixed teacher trajectory, offering a principled, condition-aware approach to one-step generation.

\section{Proofs}
\label{sec:proof}
In this section, we present the formal statement and proof of Theorem~\ref{thm:denoising-projection}, alongside the proof of Corollary~\ref{cor:affine-covariance}.

\begin{theorem}[Denoising Projection Moment Hierarchy, Formal Statement]
\label{thm:app-denoising-projection}
Fix a condition $c$ and $t\in(0,1)$. Let $X_1\sim P(\cdot\mid c)$ and $X_0\sim \mathcal N(0,I)$ be independent variables. Define the interpolant $Z=tX_1+(1-t)X_0$ and the target velocity $V=X_1-X_0$, assuming $\mathbb E[\|V\|^2\mid c]<\infty$. Let $\mathcal H = L^2(P_{Z\mid c};\mathbb R^D)$ be the Hilbert space of square-integrable vector fields, wherein the unconstrained flow-matching solution is given by $v^\star = \mathbb E[V\mid Z,c]$.

For any integer $n\ge 0$, let $\mathcal H_n \subset \mathcal H$ be the subspace of vector-valued polynomials in $Z$ up to degree $n$. The optimal degree-$n$ polynomial denoiser,
\begin{equation}
v_n^\star = \arg\min_{v\in\mathcal H_n} \mathbb E\left[\|v(Z)-V\|^2\mid c\right],
\end{equation}
satisfies $v_n^\star = \Pi_n v^\star$, where $\Pi_n$ is the orthogonal projection onto $\mathcal H_n$. If the polynomial space $\bigcup_{n\ge 0}\mathcal H_n$ is dense in $\mathcal H$, these projections converge strongly to $v^\star$, admitting the orthogonal decomposition $v^\star = \sum_{n=0}^\infty \Delta_n$ in $L^2(P_{Z\mid c})$, where $\Delta_0 = v_0^\star$ and $\Delta_n = v_n^\star - v_{n-1}^\star$.

Furthermore, by the first-order optimality over $\mathcal H_n$, $v_n^\star$ satisfies the normal equations, thereby perfectly matching the target moment tensors of the unconstrained velocity:
\begin{equation}
\mathbb E\left[ v_n^\star(Z)\otimes Z^{\otimes k} \mid c \right]
=
G_k^t(P\mid c)
:=
\mathbb E\left[ V \otimes Z^{\otimes k} \mid c \right],
\qquad
k=0,\ldots,n.
\end{equation}
Writing the $j$-th conditional data moment as $M_j(P\mid c)=\mathbb E[X_1^{\otimes j}\mid c]$, the target tensors $G_k^t$ obey the triangular recurrence relation:
\begin{equation}
G_k^t(P\mid c) = t^k M_{k+1}(P\mid c) + F_k^t(M_0,\ldots,M_k),
\end{equation}
where $F_k^t$ is a deterministic tensor-valued function depending solely on $t$, the known Gaussian moments of $X_0$, and the lower-order data moments. Consequently, the optimal degree-$n$ polynomial denoiser uniquely determines the conditional data moments up to order $n+1$.
\end{theorem}

\begin{proof}
We begin by establishing the orthogonal projection property and the resulting decomposition.

For an arbitrary $v \in \mathcal{H}$, we decompose the residual as:
\begin{equation}
v(Z) - V = \bigl(v(Z) - v^\star(Z)\bigr) + \bigl(v^\star(Z) - V\bigr).
\end{equation}
By definition, $v^\star(Z) = \mathbb{E}[V \mid Z, c]$, which immediately implies $\mathbb{E}[V - v^\star(Z) \mid Z, c] = 0$. Consequently, the cross term vanishes in expectation:
\begin{equation}
\mathbb{E}\left[ \bigl(v(Z) - v^\star(Z)\bigr)^\top \bigl(v^\star(Z) - V\bigr) \mid c \right] = 0.
\end{equation}
This decouples the mean squared error into:
\begin{equation}
\mathbb{E}\left[ \|v(Z) - V\|^2 \mid c \right] 
= 
\mathbb{E}\left[ \|v(Z) - v^\star(Z)\|^2 \mid c \right] 
+ 
\mathbb{E}\left[ \|v^\star(Z) - V\|^2 \mid c \right].
\end{equation}
Since the second term is independent of $v$, minimizing the flow-matching objective over $\mathcal{H}_n$ is mathematically equivalent to minimizing the distance $\|v - v^\star\|_{\mathcal{H}}^2$ over $v \in \mathcal{H}_n$. Thus, the minimizer is exactly the orthogonal projection: $v_n^\star = \Pi_n v^\star$.

To establish the orthogonality of the increments $\Delta_n = v_n^\star - v_{n-1}^\star$, consider any $h \in \mathcal{H}_{n-1}$. Since $v_n^\star = \Pi_n v^\star$ and $\mathcal{H}_{n-1} \subset \mathcal{H}_n$, the projection residual $v^\star - v_n^\star$ is orthogonal to $h$. Similarly, $v^\star - v_{n-1}^\star \perp h$. Subtracting these equations yields $\langle v_n^\star - v_{n-1}^\star, h \rangle_{\mathcal{H}} = 0$, implying $\Delta_n \perp \mathcal{H}_{n-1}$. Because $\Delta_m \in \mathcal{H}_m \subset \mathcal{H}_{n-1}$ for any $m < n$, we obtain $\langle \Delta_n, \Delta_m \rangle_{\mathcal{H}} = 0$. Under the assumption that the polynomial space is dense in $\mathcal{H}$, the orthogonal projections converge strongly, yielding the valid infinite series $v^\star = \sum_{n=0}^\infty \Delta_n$ in $L^2(P_{Z\mid c})$.

Having established the projection properties, we now turn to the moment hierarchy and the triangular recurrence relation.

Since $\mathcal{H}_n$ is spanned by the tensor polynomials $\{Z^{\otimes k}\}_{k=0}^n$, the condition that the residual $(v_n^\star - V)$ is orthogonal to $\mathcal{H}_n$ in $L^2(P_{Z\mid c})$ is equivalent to testing against the basis elements. Thus, for any $k \in \{0, \dots, n\}$:
\begin{equation}
\langle v_n^\star - V, z \mapsto z^{\otimes k} \rangle_{\mathcal{H}} = 0
\implies
\mathbb{E}\left[ \bigl(v_n^\star(Z) - V\bigr) \otimes Z^{\otimes k} \mid c \right] = 0.
\end{equation}
Rearranging this yields the moment matching property: $\mathbb{E}[v_n^\star(Z) \otimes Z^{\otimes k} \mid c] = \mathbb{E}[V \otimes Z^{\otimes k} \mid c] = G_k^t(P \mid c)$.

To establish the triangular recurrence, we explicitly expand $G_k^t(P \mid c)$. Recalling $V = X_1 - X_0$ and $Z = t X_1 + (1-t) X_0$, we write:
\begin{equation}
G_k^t(P \mid c) 
= 
\mathbb{E}\left[ (X_1 - X_0) \otimes \bigl(t X_1 + (1-t) X_0\bigr)^{\otimes k} \mid c \right].
\end{equation}
By the non-commutative binomial expansion of the tensor power, the highest-order term of $X_1$ inside the parentheses is $t^k X_1^{\otimes k}$. Distributing $(X_1 - X_0)$ into this expansion, the term with the highest tensor power of $X_1$ is generated solely by $X_1 \otimes (t^k X_1^{\otimes k}) = t^k X_1^{\otimes (k+1)}$.

All other terms in the expansion consist of tensor products mixing at most $k$ copies of $X_1$ with various copies of $X_0$. Because $X_0 \sim \mathcal{N}(0, I)$ is independent of $X_1$ given $c$, the expectations of these mixed terms factorize. The expectations of the $X_0$ components yield standard Gaussian moments (which are known constants), and the expectations of the $X_1$ components yield lower-order data moments $M_j(P \mid c)$ for $j \le k$. Gathering all these lower-order terms into a single deterministic tensor-valued function $F_k^t(M_0, \dots, M_k)$, we obtain:
\begin{equation}
G_k^t(P \mid c) 
= 
t^k \mathbb{E}\left[X_1^{\otimes (k+1)} \mid c\right] + F_k^t(M_0,\dots,M_k) 
= 
t^k M_{k+1}(P \mid c) + F_k^t(M_0,\dots,M_k).
\end{equation}
Since $t \in (0,1)$, the coefficient $t^k$ is strictly positive. Therefore, knowing $G_k^t(P \mid c)$ (which is perfectly matched by the degree-$n$ projection for $k \le n$) and the previously determined lower-order moments allows iterative back-substitution to uniquely solve for $M_{k+1}(P \mid c)$. This concludes the proof.
\end{proof}

\begin{corollary}[Affine Component Identifies a Covariance Operator]
\label{cor:app-affine-covariance}
Under the setting of Theorem~\ref{thm:denoising-projection}, assume $X_1\mid c$ has a finite second moment. Let $\mu(c)=\mathbb{E}[X_1\mid c]$ and $\Sigma(c)=\operatorname{Cov}(X_1\mid c)$. Then the degree-zero projection is exactly the conditional mean:
\begin{equation}
v_0^\star(z,t,c)=\mu(c),
\end{equation}
and the degree-one projection admits the affine form:
\begin{equation}
v_1^\star(z,t,c)
=
\mu(c)+R_t(c)\bigl(z-t\mu(c)\bigr),
\end{equation}
where the affine operator is given by
\begin{equation}
R_t(c)
=
\bigl[t\Sigma(c)-(1-t)I\bigr]
\bigl[t^2\Sigma(c)+(1-t)^2I\bigr]^{-1}.
\end{equation}
For a fixed $t\in(0,1)$, the mapping $\Sigma(c)\mapsto R_t(c)$ is rigorously injective. Therefore, the affine component of the denoising velocity serves as an exact, injective transformation of the covariance.
\end{corollary}

\begin{proof}
By standard least-squares theory, the optimal degree-zero predictor is simply the conditional expectation of the target:
\begin{equation}
v_0^\star(z,t,c) = \mathbb E[V\mid c] = \mu(c).
\end{equation}

Similarly, the optimal degree-one (affine) predictor is given by the exact population ordinary least squares (OLS) solution:
\begin{equation}
v_1^\star(z,t,c) = \mathbb E[V\mid c] + \operatorname{Cov}(V,Z\mid c)\operatorname{Cov}(Z\mid c)^{-1}\bigl(z-\mathbb E[Z\mid c]\bigr).
\end{equation}
Recall that $Z=tX_1+(1-t)X_0$ and $V=X_1-X_0$. Taking conditional expectations yields $\mathbb E[Z\mid c]=t\mu(c)$ and $\mathbb E[V\mid c]=\mu(c)$. Utilizing the independence of $X_1$ and $X_0\sim\mathcal N(0,I)$, we directly evaluate the covariance matrices:
\begin{equation}
\operatorname{Cov}(Z\mid c) = t^2\Sigma(c)+(1-t)^2I, 
\qquad
\operatorname{Cov}(V,Z\mid c) = t\Sigma(c)-(1-t)I.
\end{equation}
Substituting these moments into the OLS formula directly yields $v_1^\star(z,t,c) = \mu(c) + R_t(c)\bigl(z-t\mu(c)\bigr)$, where $R_t(c) = \operatorname{Cov}(V,Z\mid c)\operatorname{Cov}(Z\mid c)^{-1}$.

Finally, to establish injectivity, we isolate $\Sigma$ from $R_t$. Right-multiplying $R_t$ by $(t^2\Sigma+(1-t)^2I)$ and grouping the terms involving $\Sigma$ yields:
\begin{equation}
t(I-tR_t)\Sigma = (1-t)\bigl(I+(1-t)R_t\bigr).
\end{equation}
Observe that the matrix factor $I-tR_t = (1-t)(t^2\Sigma+(1-t)^2I)^{-1}$ is strictly positive definite, hence invertible. Left-multiplying by its inverse yields the explicit closed-form expression for $\Sigma(c)$:
\begin{equation}
\Sigma(c)
=
\frac{(1-t)}{t}
\bigl[I-tR_t(c)\bigr]^{-1}
\bigl[I+(1-t)R_t(c)\bigr].
\end{equation}
Thus, $R_t$ uniquely determines the covariance matrix.
\end{proof}

\section{Additional Details on AMFD Training}
\label{app:amfd_details}
\subsection{AMFD Training Algorithm}
\label{app:amfd_algorithm}
\begin{algorithm}[t]
\caption{Alternating Training Pipeline for AMFD}
\label{alg:amfd_training}
\begin{algorithmic}[1]
\Require Generator \(G_\theta\), frozen encoders \(\{\phi_\ell\}_{\ell=1}^L\), amortizers \(\{\psi_\ell\}_{\ell=1}^L\)
\Repeat
    \State Sample real data \((x_r,c)\) and noise \(z\); generate \(x_g=G_\theta(z,c)\)
    \State Compute feature representations \(h_r^\ell=\phi_\ell(x_r)\) and \(h_g^\ell=\phi_\ell(x_g)\) for all \(\ell=1,\ldots,L\)
    \State Update amortizers by minimizing \(\mathcal L_\mathrm{amort}({\psi})\) (Eq.~\ref{eq:amortizer-loss}), with generated features \(h_g^\ell\) detached
    \State Freeze amortizers and update generator \(G_\theta\) by minimizing \(\mathcal L_\mathrm{gen}(\theta)\) (Eq.~\ref{eq:generator-loss})
\Until{convergence}
\end{algorithmic}
\end{algorithm}
Algorithm~\ref{alg:amfd_training} outlines the complete alternating optimization procedure for the AMFD framework.

\subsection{Derivation of the Jacobi Covariance Weighting}
\label{app:jacobi-covariance}

In this section, we derive the weighting factor $\omega_{\mathrm{cov}}(t)$ introduced in Section~\ref{subsec:training_algo}. Our objective is to calibrate the covariance loss such that its gradient approximately matches the standard Fr\'{e}chet Distance force, $\frac{1}{2} d\Sigma$, at the isotropic equilibrium point $\Sigma = I$.

Crucially, at $\Sigma = I$, the operators $R_t$, $A_t$, and the Jacobi operator $\mathcal J_t = I - t R_t$ all reduce to scalar multiples of the identity matrix. Because these operators commute at this specific point, we can compute their differentials using standard scalar calculus.

By the proof of Corollary~\ref{cor:app-affine-covariance}, the covariance is expressed as:
\begin{equation}
\Sigma = \frac{1-t}{t} (I - t R_t)^{-1} (I + (1-t)R_t).
\end{equation}
Evaluating the differential $d\Sigma$ with respect to $R_t$ around the commuting point yields:
\begin{equation}
d\Sigma = \frac{1-t}{t} (I - t R_t)^{-2} \big[ (1-t)(I - t R_t) - (-t)(I + (1-t)R_t) \big] dR_t = \frac{1-t}{t} \mathcal J_t^{-2} dR_t.
\end{equation}

Using our residual parameterization $R_t = -I + tA_t$, we substitute $dR_t = t dA_t$ to obtain:
\begin{equation}
d\Sigma = (1-t) \mathcal J_t^{-2} dA_t.
\end{equation}

Our applied generator force is defined as $\delta a_{\mathrm J} = \mathcal J_t dA_t \mathcal J_t$. Since all matrices commute at $\Sigma = I$, this simplifies to $\delta a_{\mathrm J} = \mathcal J_t^2 dA_t$, which implies $dA_t = \mathcal J_t^{-2} \delta a_{\mathrm J}$. Substituting this back into the $d\Sigma$ equation provides the direct relationship:
\begin{equation}
\label{eq:app-dsigma-force}
d\Sigma = (1-t) \mathcal J_t^{-4} \delta a_{\mathrm J}.
\end{equation}

To compute the explicit scalar value, we evaluate $\mathcal J_t(I) = I - t R_t(I)$ using Eq.~\ref{eq:affine-covariance-operator}:
\begin{equation}
\mathcal J_t(I) = I - t \left( \frac{2t-1}{t^2 + (1-t)^2} I \right) = \frac{1-t}{t^2 + (1-t)^2} I.
\end{equation}
Plugging the scalar value of $\mathcal J_t(I)^{-4}$ into Eq.~\ref{eq:app-dsigma-force} yields:
\begin{equation}
d\Sigma = \frac{\bigl(t^2+(1-t)^2\bigr)^4}{(1-t)^3} \delta a_{\mathrm J}.
\end{equation}

To precisely recover the target FD force $\frac{1}{2} d\Sigma$, the loss weighting must be:
\begin{equation}
\omega_{\mathrm{cov}}(t) = \frac{1}{2} \frac{\bigl(t^2+(1-t)^2\bigr)^4}{(1-t)^3}.
\end{equation}

\section{Implementation Details}
\label{app:implementation_details}

In this section, we provide comprehensive implementation details for our experiments. Specifically, we specify the representation encoders utilized (\S\ref{app:details_encoders}), introduce a custom analytical JVP implementation designed to accelerate amortizer training (\S\ref{app:details_jvp}), outline the experimental settings for our amortizer ablation studies (\S\ref{app:details_ablation}), and provide the default amortizer settings for generation tasks (\S\ref{app:general-amortizer-configs}). Furthermore, we supply explicit data curation procedures and hyperparameter configurations for both ImageNet (\S\ref{sec:imagenet_details}) and text-to-image (\S\ref{app:details_t2i}) post-training environments.

\subsection{Representation Encoders}
\label{app:details_encoders}

Table~\ref{tab:repr_encoders} summarizes the representation encoders utilized across our experimental settings. For the pooling operations, \texttt{cls} denotes the class token, \texttt{avg} indicates global average pooling over spatial or patch tokens, and \texttt{attn} refers to the attention-based pooling head. During amortizer training, we apply per-dimension normalization to all extracted features. This ensures that the inputs to the amortizer maintain zero mean and unit variance, facilitating stable neural network optimization.

\begin{table}[h]
\centering
\caption{Detailed configuration of the representation encoders. We list the model names and architectures, the used weight checkpoints, the feature extraction pooling operations, the input image resolutions, and the resulting feature dimensions.}
\label{tab:repr_encoders}
\addtolength{\tabcolsep}{-1pt}
\scriptsize
\begin{tabular}{@{} l l c c c c @{}}
\toprule
\textbf{Model} & \textbf{Checkpoint} & \textbf{Arch.} & \textbf{Pooling} & \textbf{Resolution} & \textbf{Dim} \\
\midrule
Inception-v3~\citep{szegedy2016rethinking} & torch-fidelity \texttt{inception\_v3} & CNN &  \texttt{avg} & 299 & 2048 \\
ConvNeXt-v2~\citep{liu2022convnet} & \texttt{convnextv2\_base.fcmae\_ft\_in22k\_in1k} & CNN & \texttt{avg} & 224 & 1024 \\
SigLIP2~\citep{tschannen2025siglip} & \texttt{vit\_so400m\_patch16\_siglip\_256.v2\_webli} & ViT & \texttt{cls}/\texttt{attn} & 224 & 1152 \\
MAE~\citep{he2022masked} & \texttt{vit\_large\_patch16\_224.mae} & ViT & \texttt{cls} & 224 & 1024 \\
CLIP~\citep{radford2021learning} & \texttt{vit\_large\_patch14\_clip\_224.openai} & ViT & \texttt{cls} & 256 & 1024 \\
DINOv2~\citep{oquab2023dinov2} & \texttt{vit\_large\_patch14\_dinov2.lvd142m} & ViT & \texttt{cls} & 224 & 1024 \\
PE-Core~\citep{bolya2026perception} & \texttt{vit\_pe\_core\_large\_patch14\_336.fb} & ViT & \texttt{attn} & 224 & 1024 \\
AIM-v2~\citep{fini2025multimodal} & \texttt{aimv2\_huge\_patch14\_224.apple\_pt} & ViT & \texttt{avg} & 224 & 1536 \\
Web-SSL~\citep{fan2025scaling} & \texttt{webssl-dino1b-full2b-224} & ViT & \texttt{cls} & 224 & 1536 \\
DreamSim~\citep{fu2023dreamsim} & DINO + CLIP + OpenCLIP ensemble & ViT & \texttt{cls} & 224 & 1792 \\
\bottomrule
\end{tabular}
\end{table}

\subsection{Custom JVP Implementation for MLP Amortizer}
\label{app:details_jvp}
\begin{algorithm}[ht]
\caption{Manual JVP Forward Pass for MLP}
\label{alg:manual_jvp}
\begin{algorithmic}[1]
\Require Tangent vector $v$, condition label $c$, time $t$, amortizer $g_\theta$.
\Ensure Exact input-direction JVP
$
\dot{x}_{\mathrm{out}}
=
\left.
\frac{\partial g_\theta(s,t,c)}{\partial s}
\right|_{s=0} v
$.
\Statex

\State $x \leftarrow \mathbf{0}$
\State $\dot{x} \leftarrow v$
\State $(x,\dot{x}) \leftarrow \mathrm{LinearJVP}(x,\dot{x}; W_{\mathrm{in}}, b_{\mathrm{in}})$
\Comment{$x=W_{\mathrm{in}}x+b_{\mathrm{in}}$, $\dot{x}=W_{\mathrm{in}}\dot{x}$}

\State \textit{\% Conditioning path has zero tangent w.r.t. $s$ and is evaluated once.}
\State $y \leftarrow \mathrm{SiLU}\big(\mathrm{TimeEmbed}(t)+\mathrm{ClassEmbed}(c)\big)$
\State $(\gamma,\beta,\alpha) \leftarrow \mathrm{AdaLN}_{0}(y)$

\For{$i=0,\ldots,N-1$}
    \If{$i>0$ \textbf{and} $i \bmod K = 0$}
        \State $(\gamma,\beta,\alpha) \leftarrow \mathrm{AdaLN}_{i/K}(y)$
        \Comment{$K$ is the AdaLN switch frequency}
    \EndIf

    \State \textit{\% RMSNorm JVP, sharing the primal normalization statistics.}
    \State $m \leftarrow \mathrm{Mean}(x^2)$
    \State $r \leftarrow (m+\epsilon)^{-1/2}$
    \State $x_{\mathrm{n}} \leftarrow x \cdot r$
    \State $\dot{r} \leftarrow -r^3 \cdot \mathrm{Mean}(x\odot \dot{x})$
    \State $\dot{x}_{\mathrm{n}} \leftarrow \dot{x}\cdot r + x\cdot \dot{r}$
    \If{RMSNorm has affine weight $w_{\mathrm{n}}$}
        \State $x_{\mathrm{n}} \leftarrow x_{\mathrm{n}}\odot w_{\mathrm{n}}$
        \State $\dot{x}_{\mathrm{n}} \leftarrow \dot{x}_{\mathrm{n}}\odot w_{\mathrm{n}}$
    \EndIf

    \State \textit{\% AdaLN modulation. Shift has no input tangent.}
    \State $h \leftarrow x_{\mathrm{n}}\odot(1+\gamma)+\beta$
    \State $\dot{h} \leftarrow \dot{x}_{\mathrm{n}}\odot(1+\gamma)$

    \State \textit{\% SwiGLU MLP JVP.}
    \State $(u,\dot{u}) \leftarrow \mathrm{LinearJVP}(h,\dot{h}; W_{1}, b_{1})$
    \State $(u_1,u_2) \leftarrow \mathrm{Chunk}(u)$
    \State $(\dot{u}_1,\dot{u}_2) \leftarrow \mathrm{Chunk}(\dot{u})$

    \State $\sigma \leftarrow \mathrm{Sigmoid}(u_1)$
    \State $s \leftarrow \mathrm{SiLU}(u_1)=u_1\odot\sigma$
    \State $\dot{s} \leftarrow \dot{u}_1\odot\sigma\odot\bigl(1+u_1\odot(1-\sigma)\bigr)$

    \State $z \leftarrow s\odot u_2$
    \State $\dot{z} \leftarrow \dot{s}\odot u_2+s\odot\dot{u}_2$

    \State $(o,\dot{o}) \leftarrow \mathrm{LinearJVP}(z,\dot{z}; W_{2}, b_{2})$

    \State \textit{\% Residual gated update.}
    \State $x \leftarrow x + o\odot\alpha$
    \State $\dot{x} \leftarrow \dot{x} + \dot{o}\odot\alpha$
\EndFor

\State \textit{\% Final AdaLN + RMSNorm + linear projection.}
\State $(\gamma_f,\beta_f) \leftarrow \mathrm{FinalAdaLN}(y)$
\State $(x_{\mathrm{n}},\dot{x}_{\mathrm{n}}) \leftarrow \mathrm{RMSNormJVP}(x,\dot{x})$
\State $x \leftarrow x_{\mathrm{n}}\odot(1+\gamma_f)+\beta_f$
\State $\dot{x} \leftarrow \dot{x}_{\mathrm{n}}\odot(1+\gamma_f)$
\State $(\_,\dot{x}_{\mathrm{out}}) \leftarrow \mathrm{LinearJVP}(x,\dot{x}; W_{\mathrm{out}}, b_{\mathrm{out}})$

\State \Return $\dot{x}_{\mathrm{out}}$
\end{algorithmic}
\end{algorithm}
Evaluating the AMFD objective in Eq.~\ref{eq:generator-cov-loss-jacobi} requires frequent JVP computations through the amortization network. In practice, relying on PyTorch's generic forward-mode automatic differentiation (AD) for our MLP amortizer resulted in severe GPU underutilization, emerging as a primary training bottleneck. To circumvent this, we implement an exact analytical JVP forward pass tailored specifically to our architecture, as outlined in Algorithm~\ref{alg:manual_jvp}. By jointly propagating primal activations and tangent vectors across linear layers, RMSNorm, SiLU/SwiGLU activations, and AdaLN-modulated residual blocks, our custom implementation maximally reuses shared intermediate tensors. This approach strictly preserves the generic AD outputs up to numerical precision while yielding substantial gains in GPU utilization and overall training throughput.

\subsection{Amortizer Ablation Details}
\label{app:details_ablation}

\textbf{Amortizer architecture details.} 
The amortizer is parameterized as an AdaLN-conditioned MLP~\citep{li2024autoregressive}. A grouped AdaLN mechanism is used to minimize conditioning overhead: rather than computing AdaLN parameters at every layer, residual blocks are grouped to share a single modulation projection. Our notation \texttt{c[channels]d[depth]a[AdaLN-blocks]} reflects this structure; for instance, \texttt{c2048d16a4} applies AdaLN only 4 times across its 16 layers.

\textbf{Evaluation metrics and protocol.}
To quantify the amortization accuracy, we evaluate its estimates of the conditional mean and the full covariance. Let $\mu_c \in \mathbb{R}^{D}$ and $\Sigma_c \in \mathbb{R}^{D \times D}$ denote the empirical mean and covariance under condition $c$, and let $\hat{\mu}_c$ and $\hat{\Sigma}_c$ be their respective estimates. We define the mean estimation error as
\begin{equation}
    \mathrm{MSE}_{\mu}
    =
    \frac{1}{|\mathcal{C}|D}
    \sum_{c \in \mathcal{C}}
    \left\|\hat{\mu}_c-\mu_c\right\|_2^2.
\end{equation}
To evaluate the covariance estimates, we materialize the complete covariance
operator by applying it to the canonical basis vectors. Because the learned
operator is not explicitly constrained to be symmetric, we extract its
symmetric component
\begin{equation}
    \hat{\Sigma}^{\mathrm{sym}}_c
    =
    \frac{1}{2}
    \left(\hat{\Sigma}_c+\hat{\Sigma}_c^\top\right),
\end{equation}
and define the covariance estimation error as
\begin{equation}
    \mathrm{MSE}_{\Sigma}
    =
    \frac{1}{|\mathcal{C}|D^2}
    \sum_{c \in \mathcal{C}}
    \left\|\hat{\Sigma}^{\mathrm{sym}}_c-\Sigma_c\right\|_F^2.
\end{equation}

For our ImageNet experiments, $\mathcal{C}$ denotes the set of evaluated classes. We compute the mean estimation error $\mathrm{MSE}_{\mu}$ across all 1,000 classes (i.e., $|\mathcal{C}| = 1000$). However, materializing and evaluating the full $D \times D$ covariance operators across all classes is computationally demanding. Therefore, we evaluate $\mathrm{MSE}_{\Sigma}$ on a subset of the 100 most populated classes ($|\mathcal{C}| = 100$), each containing exactly 1,300 images. This subset corresponds to class indices 0--42, 44--50, 52--61, 63--97, 99--102, and 104. For both metrics, we report the macro-average over their respective evaluated classes. In contrast, for the unconditional (marginal) experiments, we compute both $\mathrm{MSE}_{\mu}$ and $\mathrm{MSE}_{\Sigma}$ directly over the global data distribution, treating the entire dataset as a single entity without class-wise partitioning.

The symmetric covariance error is inherently better aligned with our generator objective than the MSE of the raw, potentially asymmetric operator. Specifically, the covariance term in the generator loss is evaluated through quadratic forms of the type $u^\top(\hat{\Sigma}_g-\hat{\Sigma}_r)u$. For any skew-symmetric matrix $K$, $u^\top K u=0$; hence, the antisymmetric component of the estimated operator has no impact on this generator objective. Consequently, $\mathrm{MSE}_{\Sigma}$ precisely measures the covariance component that can influence the generator update, properly avoiding penalties on antisymmetric estimation artifacts that are invisible to the generator loss.

\textbf{Training configurations.} Table~\ref{tab:ablation_settings} details the default hyperparameter configurations for these ablation studies.

\begin{table}[h]
\centering
\caption{Default hyperparameters for the amortizer ablation studies.}
\label{tab:ablation_settings}
\small
\begin{tabular}{@{} l c @{}}
\toprule
representation encoder & Inception-v3 \texttt{avg} \\
amortizer architecture & AdaLN MLP \\
amortization variant & AMFD-C (conditional) \\
task conditioning time ($t$) & real: $t=2$ for mean, $t\in(0,1)$ for cov; generated: $-t$ \\
optimizer & AdamW, $\beta_1=0.9, \beta_2 = 0.95$ \\
learning rate & $1\text{e-}4$ \\
weight decay & $0$ \\
precision & bf16 \\
batch Size & 1024 \\
training steps & 20,000 \\
evaluation metrics & $\mathrm{MSE}_\mu$, $\mathrm{MSE}_\Sigma$ \\
\bottomrule
\end{tabular}
\end{table}

\subsection{Default Amortizer Configurations for Image Generation}
\label{app:general-amortizer-configs}

Unless otherwise specified, the following settings apply universally to all amortizer training procedures:
\begin{itemize}
    \item \textbf{Architecture \& initialization:} Parameterized by a \texttt{c2048d16a4} MLP with random initialization.
    \item \textbf{Condition Injection:} The external conditioning signal is explicitly injected into the amortizer network, corresponding to the AMFD-C configuration.
    \item \textbf{Optimization:} We adopt the same optimizer and global batch size as the generator, applying a constant learning rate of $1\text{e-}4$ without weight decay.
    \item \textbf{Diffusion parameters:} The interpolation time is fixed at $t=0.25$. Given the lightweight architecture of the MLP amortizer, we independently draw 4 noise samples per data input to expand the effective batch size and fully utilize hardware computational capacity.
    \item \textbf{Task-conditioning time:} For the real branch, it is set to $t=2$ for mean prediction and $t \in (0,1)$ for covariance prediction. The generated branch utilizes the corresponding negative time steps ($-t$).
    \item \textbf{JVP Calculation:} We employ a custom analytical implementation for the Jacobian-vector product (JVP), as detailed in Appendix~\ref{app:details_jvp}.
    \item \textbf{Multi-Representation Normalization:} To compute the normalized generator loss (Eq.~\ref{eq:normalized-generator-loss}), we consistently fix $\rho=1$ and $\epsilon=0.01$, aligning with the multi-representation FD-loss formulation in~\citet{yang2026representation}. Notably, this normalization is bypassed during single-representation training.
\end{itemize}
For our default MLP amortizers, we utilize the custom analytical JVP implementation. For experiments conducted in native generative spaces where the amortizers adopt Transformer architectures, we opt to employ PyTorch's functional API (\texttt{torch.func.jvp} and \texttt{vmap}). While our analytical MLP JVP inherently avoids distributed synchronization issues, the PyTorch JVP engine can conflict with the asynchronous communication hooks of Distributed Data Parallel (DDP). To ensure a unified, robust codebase compatible across all amortizer architectures, we systematically structure our distributed training pipeline by wrapping the JVP operation strictly inside the DDP module (\textit{i.e.}, \texttt{DDP(jvp(Amortizer))}). This uniform architectural arrangement natively circumvents synchronization failures during forward-mode differentiation, similar to the strategy noted in \citet{lin2026continuous}. Moreover, we empirically observed that computing JVPs through \texttt{LayerNorm} leads to severe training instability. Conversely, \texttt{RMSNorm} provides significantly better stability under such operations. While the majority of our adopted amortizer architectures natively employ \texttt{RMSNorm}, the SiT~\citep{ma2024sit} model used in our native space experiments defaults to \texttt{LayerNorm}. To ensure stable training, we explicitly replace all \texttt{LayerNorm} layers with \texttt{RMSNorm} exclusively within the SiT \textit{amortizer}, while leaving the SiT \textit{generator} entirely unmodified to preserve its exact pretrained formulation. This issue has also been reported by \citet{zhou2025terminal} and \citet{lin2026continuous}.
\begin{table}[t]
\centering
\caption{Configurations for ImageNet-256$\times$256 post-training. The top section details post-training using representation features, while the bottom section specifies settings for four native generative spaces (pixel, SD-VAE~\citep{rombach2022high}, VA-VAE~\citep{yao2025reconstruction}, and RAE~\citep{zheng2025diffusion}). For native spaces, we adopt JiT~\citep{li2026back}, SiT~\citep{ma2024sit}, LightningDiT~\citep{yao2025reconstruction}, and DiT$^\text{DH}$ as architectures for both the generator and the amortizer.}
\label{tab:imagenet_hyperparams}
\small
\addtolength{\tabcolsep}{-5.5pt}
\begin{tabular}{@{} l c c c c @{}}
\toprule

\textit{Representation feature spaces} & \multicolumn{2}{c}{\makebox[5cm][c]{\textbf{JiT}}} & \multicolumn{2}{c}{\makebox[5cm][c]{\textbf{pMF}}} \\
\midrule
model sizes & \multicolumn{2}{c}{B, L, H} & \multicolumn{2}{c}{B, L, H} \\
initialization & \multicolumn{4}{c}{official pretrained weights} \\
AMFD-Incep. & \multicolumn{4}{c}{Inception-v3 \texttt{avg}} \\
AMFD-SIM & \multicolumn{4}{c}{SigLIP2 \texttt{cls} + Inception-v3 \texttt{avg} + MAE \texttt{cls}} \\
batch size & \multicolumn{4}{c}{1024} \\
optimizer & \multicolumn{4}{c}{AdamW~\citep{loshchilov2017decoupled}, $\beta_1=0.9, \beta_2=0.95$} \\
weight decay & \multicolumn{4}{c}{0} \\
learning rate & \multicolumn{2}{c}{$1\text{e-}5$} & \multicolumn{2}{c}{$1\text{e-}6$} \\
lr schedule & \multicolumn{4}{c}{cosine} \\
total epochs (Table~\ref{tab:system-level-amfd}) & \multicolumn{4}{c}{100} \\
total epochs (Figure~\ref{fig:jitb_pretrain} \& \ref{fig:jitl_pretrain}) & \multicolumn{4}{c}{20} \\
warmup epochs & \multicolumn{4}{c}{1} \\
precision & \multicolumn{4}{c}{bf16 mixed} \\
dropout & \multicolumn{4}{c}{0} \\
augmentation & \multicolumn{4}{c}{center crop, horizontal flip} \\
EMA decay & \multicolumn{4}{c}{EDM2-style~\citep{karras2024analyzing}} \\

\midrule
\textit{Native generative spaces} & \makebox[2.5cm][c]{\textbf{Pixel}} & \makebox[2.5cm][c]{\textbf{SD-VAE}} & \makebox[2.5cm][c]{\textbf{VA-VAE}} & \makebox[2.5cm][c]{\textbf{RAE}} \\
\midrule
generator model & JiT-L/16 & SiT-XL/2 & LightningDiT-XL/1 & DiT$^{\text{DH}}$-XL/2 \\
amortizer model & JiT-B/16 & SiT-B/2 & LightningDiT-B/1 & DiT$^{\text{DH}}$-B/2 \\
initialization & \multicolumn{4}{c}{official pretrained weights} \\
batch size & \multicolumn{4}{c}{1024} \\
optimizer & \multicolumn{4}{c}{AdamW~\citep{loshchilov2017decoupled}, $\beta_1=0.9, \beta_2=0.95$} \\
weight decay & \multicolumn{4}{c}{0} \\
learning rate & \multicolumn{4}{c}{$1\text{e-}5$} \\
lr schedule & \multicolumn{4}{c}{cosine} \\
total epochs (Figure~\ref{fig:native_space}) & \multicolumn{4}{c}{20} \\
warmup epochs & \multicolumn{4}{c}{1} \\
precision & \multicolumn{4}{c}{bf16 mixed} \\
dropout & \multicolumn{4}{c}{0} \\
augmentation & \multicolumn{4}{c}{center crop, horizontal flip} \\
EMA decay & \multicolumn{4}{c}{EDM2-style~\citep{karras2024analyzing}} \\
amortizer noise samples & \multicolumn{4}{c}{1} \\
amortizer JVP calculation & \multicolumn{4}{c}{\texttt{torch.func.jvp}} \\
\bottomrule
\end{tabular}
\end{table}


\subsection{ImageNet Post-Training}
\label{sec:imagenet_details}

\textbf{Training settings.} Table~\ref{tab:imagenet_hyperparams} details the comprehensive hyperparameter configurations for ImageNet post-training, covering both the representation-feature approaches (employing JiT and pMF) and the investigation across various native generative spaces.

\textbf{Evaluation.} We measure the model performance using the standard FID alongside the recently introduced FDr$^6$ metric~\citep{yang2026representation}.

\subsection{Text-to-Image Post-Training}
\label{app:details_t2i}

\textbf{Training settings.} We evaluate AMFD in both the pixel space using PixelGen~\citep{ma2026pixelgen} and the VAE latent space using FLUX.2 [klein] 4B Base~\citep{blackforestlabs2026flux2klein} and FLUX.2 [klein] 4B~\citep{blackforestlabs2026flux2klein}. i) To construct the training corpus for the PixelGen experiments, we aggregate subsets from BLIP3o-60k~\citep{chen2025blip3} ($58,859$ samples), ShareGPT-4o-Image~\citep{chen2025sharegpt} ($45,717$ samples), and Echo-4o-Image~\citep{ye2025echo} (excluding the multi-reference split, $105,506$ samples). This curation yields a total of $210,082$ training samples. During post-training, the additional transformer blocks introduced for text conditioning remain frozen. ii) For the FLUX.2 [klein] 4B models, our data preparation protocol and the selection of the ten representation encoders generally follow the settings introduced by iRDM~\citep{feng2026representation}. However, as DINOv3 is currently inaccessible, we substitute DINOv3-L with DINOv2-L. The reference dataset is partitioned into two distinct subsets: a \textit{perception block} based on natural COCO captions, and a \textit{composition block} based on detector-verified GenEval prompts. All reference images are generated offline at a $512\times512$ resolution using the four-step FLUX.2 [klein] 4B teacher. Unlike iRDM, which truncates captions to $48$ tokens, we preserve the original $512$ text token length of FLUX.2 [klein]. For the perception block, we generate $24$ candidates for each of the $82,783$ COCO \texttt{train2014} captions and retain the top three scored by PickScore. This exactly matches the iRDM dataset scale, yielding $248,349$ samples. For the composition block, given the $553$ GenEval prompts, we sample $150$ seeds (expanding up to $1,000$ seeds if necessary) to retain up to $100$ detector-verified generations per prompt. Notably, two specific prompts (one targeting position and one color attribute) yielded zero valid samples even after exploring $1,000$ seeds. Consequently, our final composition block comprises $53,504$ samples, closely approximating iRDM's $53,800$. In total, the FLUX.2 [klein] training corpus contains $301,853$ samples. The complete training configurations for PixelGen and FLUX.2 [klein] experiments are summarized in Table~\ref{tab:t2i_hyperparams}.

\textbf{Evaluation.} We evaluate our models on the GenEval~\citep{ghosh2023geneval} and PickScore~\citep{kirstain2023pick} benchmarks. For PickScore, we report the mean score over the $499$ prompts from the Pick-a-Pic test set. Regarding the baseline results in Table~\ref{tab:text-to-image}, the GenEval metrics for FLUX.2 [klein] 4B and its post-training methods DMD2 and iRDM are directly cited from the iRDM paper~\citep{feng2026representation}. For PickScore, since the specific prompt set evaluated in their paper is not publicly available, we directly cite the baseline scores reported in their official codebase, which are based on the Pick-a-Pic test set. The PickScore for DMD2 is left blank, as it is not available in the repository.

\begin{table}[t]
\centering
\caption{Hyperparameter configurations for text-to-image post-training. The top section details post-training in pixel space (PixelGen), while the bottom section specifies settings for VAE space (FLUX.2 [klein]).}
\label{tab:t2i_hyperparams}
\small
\addtolength{\tabcolsep}{-2pt}
\begin{tabular}{@{} l c c @{}}
\toprule

\textit{Pixel-space generation} & \multicolumn{2}{c}{\textbf{PixelGen}} \\
\midrule
original sampling steps & \multicolumn{2}{c}{$25\times2\times2$} \\
initialization & \multicolumn{2}{c}{official checkpoint} \\
base resolution & \multicolumn{2}{c}{$512\times512$} \\
trainable parameters & \multicolumn{2}{c}{image generator only} \\
text encoder (frozen) & \multicolumn{2}{c}{Qwen3-1.7B} \\
text tokens & \multicolumn{2}{c}{128} \\
AMFD-SIM & \multicolumn{2}{c}{SIM (Inception \texttt{avg}, SigLIP2 \texttt{cls}, MAE \texttt{cls})} \\
training data & \multicolumn{2}{c}{public datasets (\S\ref{app:details_t2i})} \\
batch size & \multicolumn{2}{c}{1024} \\
optimizer & \multicolumn{2}{c}{AdamW~\citep{loshchilov2017decoupled}, $\beta_1=0.9, \beta_2=0.95$} \\
learning rate & \multicolumn{2}{c}{$1\text{e-}5$} \\
total steps & \multicolumn{2}{c}{5000} \\
warmup steps & \multicolumn{2}{c}{500} \\
precision & \multicolumn{2}{c}{bf16} \\
EMA decay & \multicolumn{2}{c}{none} \\
amortizer architecture & \multicolumn{2}{c}{\texttt{c2048d16a4} MLP w/ 8-head cross-attn per adaLN} \\

\midrule
\textit{Latent-space generation} & \makebox[4.5cm][c]{\textbf{FLUX.2 [klein] 4B Base}} & \makebox[4.5cm][c]{\textbf{FLUX.2 [klein] 4B}} \\
\midrule
original sampling steps & $50\times2$ & 4 \\
initialization & \multicolumn{2}{c}{official pretrained weights} \\
base resolution & \multicolumn{2}{c}{$512\times512$} \\
trainable parameters & \multicolumn{2}{c}{image generator only} \\
text encoder (frozen) & \multicolumn{2}{c}{Qwen3-4B} \\
text tokens & \multicolumn{2}{c}{512} \\
AMFD-SIM & \multicolumn{2}{c}{SIM (Inception \texttt{avg}, SigLIP2 \texttt{cls}, MAE \texttt{cls})} \\
AMFD-10 enc. & \multicolumn{2}{c}{10 encoders in Table~\ref{tab:repr_encoders}, SigLIP2 using \texttt{attn}} \\
training data & \multicolumn{2}{c}{prepared following~\citet{feng2026representation}} \\
batch size & \multicolumn{2}{c}{1024} \\
optimizer & \multicolumn{2}{c}{AdamW~\citep{loshchilov2017decoupled}, $\beta_1=0.9, \beta_2=0.95$} \\
learning rate & \multicolumn{2}{c}{$5\text{e-}6$} \\
total steps & \multicolumn{2}{c}{1500} \\
warmup Steps & \multicolumn{2}{c}{150} \\
precision & \multicolumn{2}{c}{bf16} \\
EMA decay & \multicolumn{2}{c}{none} \\
amortizer architecture & \multicolumn{2}{c}{\texttt{c2048d16a4} MLP w/ 8-head cross-attn per adaLN} \\
\bottomrule
\end{tabular}
\end{table}

\section{Additional Metrics on ImageNet Post-training}
\label{app:add_metrics}
Table~\ref{tab:add_metrics} reports the FDr~\citep{yang2026representation} metrics evaluated across six diverse pre-trained representation encoders for ImageNet post-training.

\begin{table}[t]
    \centering
    \caption{Additional FDr metrics for one-step post-training on ImageNet-256$\times$256. Lower is better for all metrics.}
    \label{tab:add_metrics}
    \small
    \addtolength{\tabcolsep}{-1pt}
    \begin{tabular}{llcccccccc}
        \toprule
        & & \multicolumn{6}{c}{FDr $\downarrow$} & & \\
        \cmidrule(lr){3-8}
        Method & Repr. & Inception & ConvNeXt & DINOv2 & MAE & SigLIP & CLIP
          & FID $\downarrow$ & FDr$^6$ $\downarrow$ \\
        \midrule
        \multicolumn{2}{l}{\color{gray!70} 50K validation images} & \color{gray!70} 1.00 & \color{gray!70} 1.00 & \color{gray!70} 1.00 & \color{gray!70} 1.00 & \color{gray!70} 1.00 & \color{gray!70} 1.00 & \color{gray!70} 1.68 & \color{gray!70} 1.00 \\
        \midrule
        JiT-B/16 & - & - & - & - & - & - & - & 3.71 & 15.65 \\
        + FD-loss & SIM & - & - & - & - & - & - & 1.00 & 5.53 \\
        + AMFD-C & SIM & 0.56 & 1.14 & 5.38 & 2.16 & 5.77 & 13.47 & 0.95 & 4.75 \\
        + AMFD-U & SIM & 0.56 & 1.19 & 4.64 & 1.28 & 3.26 & 12.54 & 0.95 & 3.91 \\
        \midrule
        JiT-L/16 & - & 1.54 & 3.49 & 6.10 & 8.07 & 19.37 & 25.82 & 2.59 & 10.73 \\
        + FD-loss & SIM & - & - & - & - & - & - & 0.77 & 3.24 \\
        + AMFD-C & SIM & 0.52 & 0.88 & 2.73 & 0.62 & 3.33 & 7.94 & 0.87 & 2.67 \\
        + AMFD-U & SIM & 0.51 & 0.92 & 2.42 & 0.31 & 1.94 & 6.03 & 0.85 & 2.02 \\
        \midrule
        JiT-H/16 & - & 1.18 & 2.52 & 4.28 & 5.65 & 11.91 & 20.40 & 1.97 & 7.66 \\
        + FD-loss & SIM & \textbf{0.45} & 0.86 & 2.10 & 0.43 & \textbf{1.68} & 10.37 & \textbf{0.75} & 2.65 \\
        + AMFD-C & SIM & 0.51 & 0.84 & 2.12 & 0.38 & 2.73 & 6.30 & 0.85 & 2.15 \\
        + AMFD-U & SIM & 0.49 & \textbf{0.79} & \textbf{1.88} & \textbf{0.21} & 1.73 & \textbf{5.65} & 0.83 & \textbf{1.79} \\
        \midrule
        pMF-B/16 & - & - & - & - & - & - & - & 3.31 & 13.70 \\
        + FD-loss & SIM & - & - & - & - & - & - & 0.85 & 3.50 \\
        + AMFD-C & SIM & 0.57 & 0.81 & 4.34 & 2.20 & 6.98 & 8.78 & 0.95 & 3.94 \\
        + AMFD-U & SIM & 0.55 & 0.83 & 3.87 & 1.69 & 5.80 & 7.86 & 0.92 & 3.43 \\
        \midrule
        pMF-L/16 & - & 1.62 & 1.36 & 6.70 & 9.72 & 20.34 & 14.81 & 2.72 & 9.09 \\
        + FD-loss & SIM & \textbf{0.47} & 0.57 & 2.21 & 0.56 & \textbf{3.03} & 5.68 & \textbf{0.78} & 2.09 \\
        + AMFD-C & SIM & 0.52 & \textbf{0.53} & 2.38 & 0.67 & 4.07 & 5.35 & 0.88 & 2.25 \\
        + AMFD-U & SIM & 0.51 & 0.55 & \textbf{2.14} & \textbf{0.47} & 3.36 & \textbf{5.03} & 0.86 & \textbf{2.01} \\
        \midrule
        pMF-H/16 & - & 1.37 & 1.15 & 5.43 & 6.25 & 15.33 & 11.68 & 2.29 & 6.87 \\
        + FD-loss & SIM & \textbf{0.46} & \textbf{0.57} & \textbf{1.74} & 0.35 & \textbf{2.46} & 5.77 & \textbf{0.77} & 1.89 \\
        + AMFD-C & SIM & 0.51 & 0.58 & 1.91 & 0.41 & 3.47 & 4.71 & 0.86 & 1.93 \\
        + AMFD-U & SIM & 0.50 & 0.60 & \textbf{1.74} & \textbf{0.27} & 2.86 & \textbf{4.54} & 0.85 & \textbf{1.75} \\
        \bottomrule
    \end{tabular}%
\end{table}

\section{Additional Visualizations}
\label{app:visualization}

\textbf{Class-conditional generation on ImageNet.} We provide additional generated samples on ImageNet 256$\times$256 using post-trained JiT-H/16 and pMF-H/16 in Figure~\ref{fig:jit-h-fd-amfd} and Figure~\ref{fig:pmf-h-fd-amfd}, respectively. In each figure, the columns from left to right display samples from base models post-trained with FD-loss, our conditional AMFD (AMFD-C), and our unconditional AMFD (AMFD-U). To ensure a direct visual comparison, identically positioned samples across the three columns are generated using the exact same initial noise.

\textbf{Text-to-image generation.} Figure~\ref{fig:add-t2i} showcases additional text-to-image samples produced by the FLUX.2 [klein] 4B model post-trained with our AMFD loss. Ordered column by column from left to right (and top to bottom within each column), the corresponding text prompts are as follows:
\begin{enumerate}[label=(\arabic*)]
    \item A striking close-up portrait of a young woman with vitiligo wearing a simple cobalt silk blouse, natural curls framing her face, confident calm expression, pale peach studio background, soft beauty lighting, truthful unretouched skin detail.
    \item A large bold word AMFD painted in glowing anime bubble letters across the center of a night festival stage, fireworks behind it, colorful spotlights, dynamic celebratory composition.
    \item A silver electric concept coupe parked on a mirror-like salt flat at blue hour, seamless aerodynamic body, thin white running lights, distant mountains doubled in shallow water, minimal luxury automotive photography.
    \item A poetic close-up portrait of a young adult Chinese woman standing among tall silver grass on a coastal hill, cream scarf, hair lifted by sea wind, overcast horizon softened behind, introspective expression, muted cinematic color and fine film grain.
    \item A sculptural ceramic vase filled with one branch of white magnolia, placed on a travertine pedestal against a muted terracotta wall, clean side light and subtle shadows, refined design editorial.
    \item A macro still life of a cracked blue ceramic bowl repaired with gold kintsugi, placed on dark linen with white flowers.
    \item A delicate watercolor-style scene of a sailboat crossing a lake at sunrise, reeds, ducks, pale orange sky, calm composition.
    \item A tiny hummingbird hovering beside a red trumpet flower, iridescent throat flashing emerald and magenta, wings rendered as delicate motion, dew-covered garden in soft morning bokeh, macro wildlife photograph.
    \item A curious alpaca standing at the edge of a high Andean village, woven red tassel on its halter, terraced green mountains and low clouds behind, bright documentary travel photograph.
    \item A secluded tropical cove viewed through limestone arches, luminous aquamarine water, a narrow white beach and one wooden boat, humid morning haze, realistic Southeast Asian travel photograph.
    \item A jewel-toned fruit tart beneath a glass bakery dome, concentric raspberries, kiwi and mango with a mirror glaze, soft reflections on a walnut counter, elegant close food photography.
    \item A proud Friesian horse galloping along a wide Atlantic beach, black mane streaming, hooves scattering silver surf, overcast sky and distant dunes, dramatic equine photography.
    \item A massive green word FRESH placed above a farm stand with fruit crates, morning sunlight, clean commercial photography.
    \item A red electric motorcycle parked in a narrow alley after rain, neon signs reflected in wet pavement, steam rising from vents, realistic urban night photography, low angle lens.
    \item A delicate close-up portrait of a young Chinese ink painter in a bright studio, subtle charcoal marks on one fingertip and cheek, long dark hair loosely tied, rice paper bokeh, introspective gaze, soft high-key natural photography.
    \item A small blue delivery van parked beside a quiet seaside road under a bright but soft afternoon sky. The van has rounded corners, clean windows, slightly worn tires, and a simple roof rack carrying two neatly tied wooden crates. The road curves gently along the coast, with a low stone barrier, pale grass, and calm water visible beyond it. A few distant sailboats appear near the horizon, and the sunlight gives the van a mild highlight along its side panel. The composition should feel clear and pleasant, like a travel photograph from a slow coastal route. Focus on the van's shape, paint surface, glass reflections, tire texture, road markings, stone barrier, sea color, and open air. The mood should be relaxed, bright, and realistic, with simple forms and a clean sense of place. The camera frames the van from a relaxed roadside angle, giving the rounded body, roof rack, crates, tires, glass, road curve, stone barrier, grass, sea, and horizon a clear readable structure. Colors stay bright and gentle, with blue paint, pale road, warm grass, soft sky, and calm water creating a clean coastal mood.
    \item A cinematic portrait of a weathered polar expedition pilot inside an icy cockpit, frost on the glass, amber instrument lights reflecting on the face, shallow depth of field, realistic skin texture, dramatic cold atmosphere.
    \item A green sea turtle drifting above a colorful coral garden, small reef fish surrounding its shell, clear sunbeams descending from the surface, richly detailed but natural underwater photography.
    \item A black cat sitting on the hood of a vintage cream convertible parked beside a lavender field, sunset sky, chrome reflections, playful fashion campaign mood.
    \item The Namib Desert at sunrise seen from a high dune, sweeping orange ridges casting blue shadows, a line of gemsbok tracks crossing untouched sand, minimal and precise fine-art landscape photograph.
    \item A sunlit close-up portrait of a weathered Greek fisherman at a harbor, white beard, knitted navy cap, deep smile lines, turquoise boat paint softly visible behind him, direct honest gaze, Mediterranean color, high-detail environmental portrait.
    \item A colorful cereal box on a breakfast table, the front label clearly reads PIXEL CRUNCH in playful letters, morning sunlight, realistic product photography.

\end{enumerate}

\begin{figure}[!t]
    \centering
    \includegraphics[width=\linewidth]{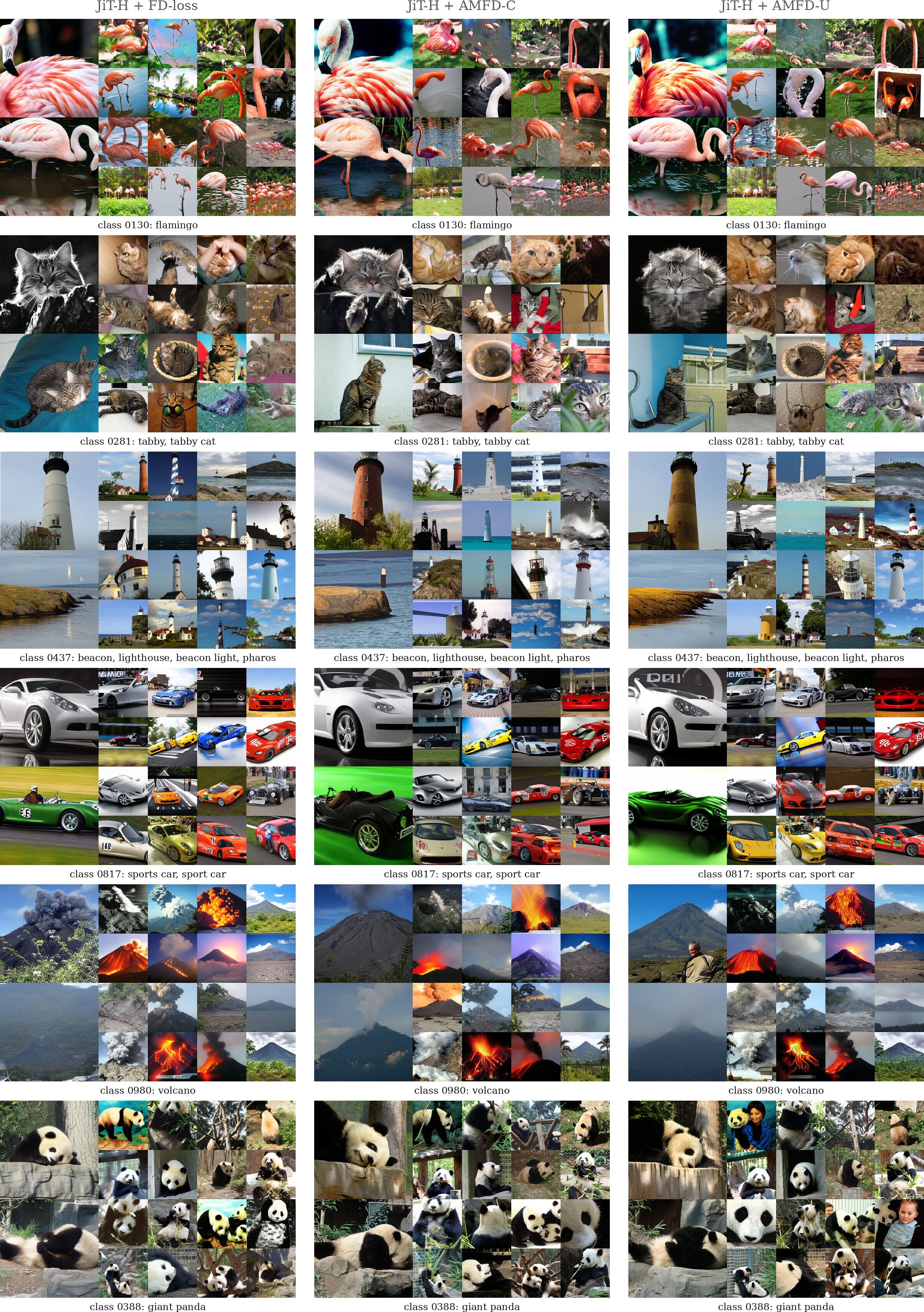}
    \caption{Uncurated samples from JiT-H post-trained with FD-loss (left), our AMFD-C (middle), and AMFD-U (right), generated using identical initial noise.}
    \label{fig:jit-h-fd-amfd}
\end{figure}

\begin{figure}[!t]
    \centering
    \includegraphics[width=\linewidth]{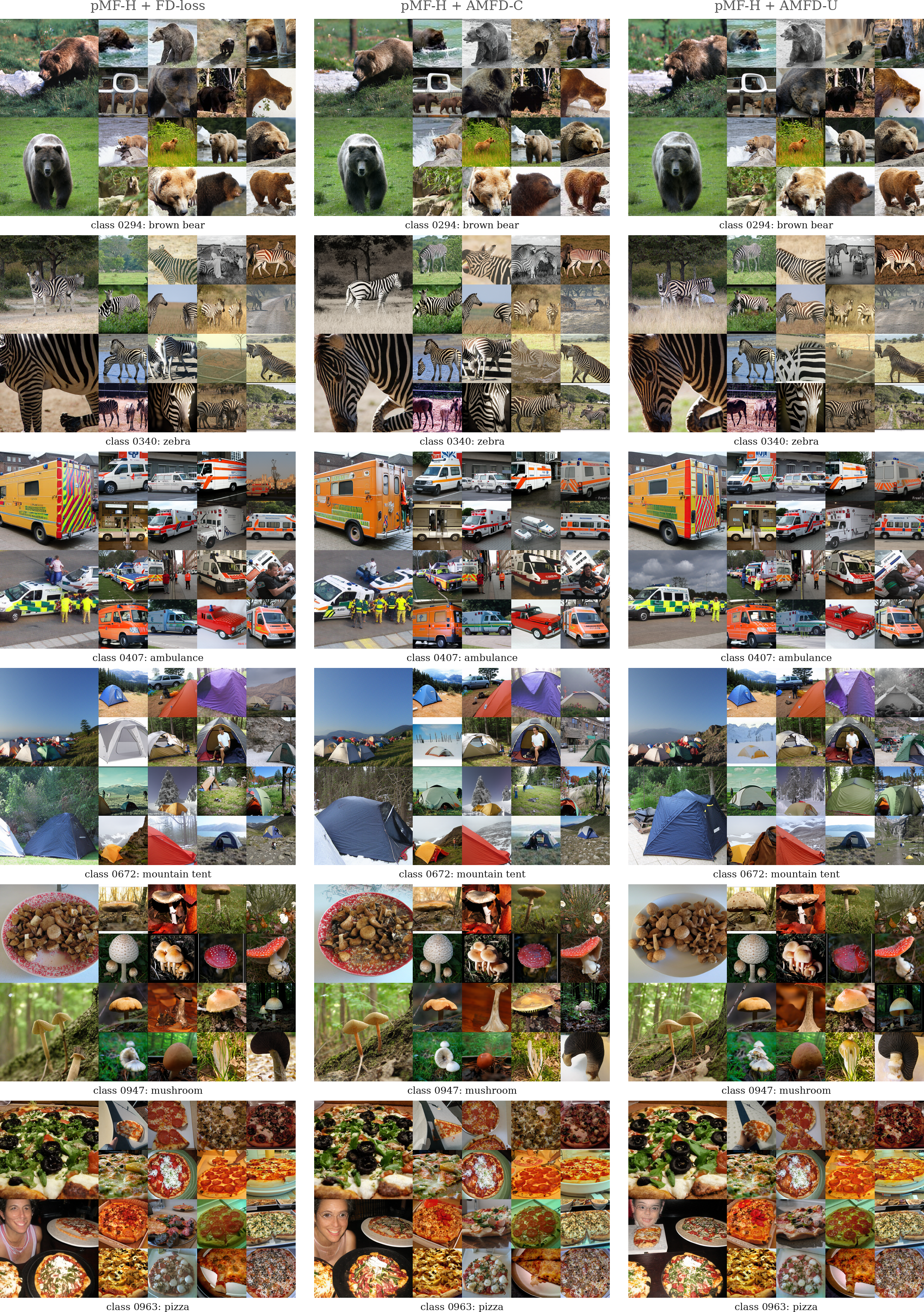}
    \caption{Uncurated samples from pMF-H post-trained with FD-loss (left), our AMFD-C (middle), and AMFD-U (right), generated using identical initial noise.}
    \label{fig:pmf-h-fd-amfd}
\end{figure}

\begin{figure}[!t]
    \centering
    \includegraphics[width=\linewidth]{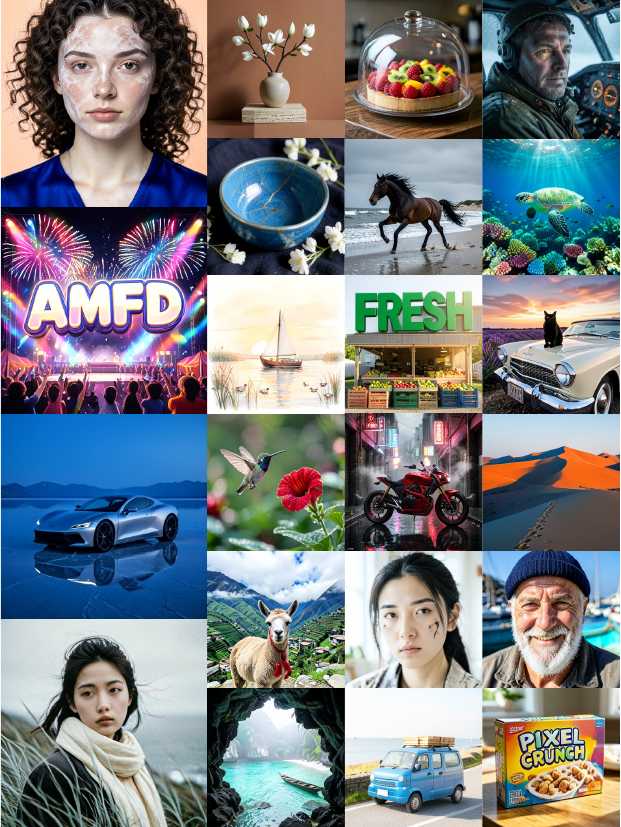}
    \caption{Additional text-to-image samples generated by the FLUX.2 [klein] 4B model post-trained with AMFD.}
    \label{fig:add-t2i}
\end{figure}
\end{document}

%% file: math_commands.tex
\usepackage{amsmath,amsfonts,bm}

\def\eqref#1{equation~\ref{#1}}

\def\1{\bm{1}}

\DeclareMathAlphabet{\mathsfit}{\encodingdefault}{\sfdefault}{m}{sl}
\SetMathAlphabet{\mathsfit}{bold}{\encodingdefault}{\sfdefault}{bx}{n}

